\documentclass[letterpaper]{article} 
\usepackage{aaai24}  
\usepackage{times}  
\usepackage{helvet}  
\usepackage{courier}  
\usepackage[hyphens]{url}  
\usepackage{graphicx} 
\usepackage{natbib}  
\usepackage{caption} 
\usepackage{algorithm}
\usepackage{algorithmic}

\usepackage{newfloat}
\usepackage{listings}
\DeclareCaptionStyle{ruled}{labelfont=normalfont,labelsep=colon,strut=off} 
\floatstyle{ruled}
\newfloat{listing}{tb}{lst}{}
\floatname{listing}{Listing}
\usepackage{bm}
\usepackage{amsmath,amsfonts,amssymb,amsthm}
\usepackage{subfigure}
\usepackage{booktabs}
\usepackage{multirow}
\usepackage{multicol}
\usepackage{graphicx}
\usepackage{lipsum} 
\usepackage[utf8]{inputenc}
\usepackage{kotex}
\usepackage{siunitx}
\usepackage{rotating}
\usepackage{multirow}
\usepackage{mathabx}
\usepackage{pifont}
\newcommand{\cmark}{\ding{51}}%
\newcommand{\xmark}{\ding{55}}%

\theoremstyle{plain}
\newtheorem{theorem}{Theorem}
\newtheorem*{customthm}{Theorem 1}

\theoremstyle{definition}
\newtheorem{definition}[theorem]{Definition}

\theoremstyle{remark}

\usepackage[textsize=small]{todonotes}
\usepackage{marginnote}
 
\title{An Attentive Inductive Bias for Sequential Recommendation \\beyond the Self-Attention}
\author{
    Yehjin Shin\equalcontrib,
    Jeongwhan Choi\equalcontrib,
    Hyowon Wi,
    Noseong Park
}
\affiliations{
    Yonsei University, Seoul, South Korea

    \{yehjin.shin, jeongwhan.choi, wihyowon, noseong\}@yonsei.ac.kr
}

\usepackage{bibentry}
 
\begin{document}

\maketitle

\begin{abstract}
Sequential recommendation (SR) models based on Transformers have achieved remarkable successes. The self-attention mechanism of Transformers for computer vision and natural language processing suffers from the oversmoothing problem, i.e., hidden representations becoming similar to tokens. In the SR domain, we, for the first time, show that the same problem occurs. We present pioneering investigations that reveal the low-pass filtering nature of self-attention in the SR, which causes oversmoothing. To this end, we propose a novel method called \textbf{B}eyond \textbf{S}elf-\textbf{A}ttention for Sequential \textbf{Rec}ommendation (BSARec), which leverages the Fourier transform to i) inject an inductive bias by considering fine-grained sequential patterns and ii) integrate low and high-frequency information to mitigate oversmoothing. Our discovery shows significant advancements in the SR domain and is expected to bridge the gap for existing Transformer-based SR models. We test our proposed approach through extensive experiments on 6 benchmark datasets. The experimental results demonstrate that our model outperforms 7 baseline methods in terms of recommendation performance. Our code is available at \url{https://github.com/yehjin-shin/BSARec}.
\end{abstract}

\section{Introduction}
Recommender systems play a vital role in web applications, delivering personalized item recommendations by analyzing user-item interactions~\cite{Rex2018pinsage,lee2018goccf,He20LightGCN,choi2021ltocf,kong2022hmlet,hong2022timekit,choi2023bspm,choi2023rdgcl,gao2023surveyrec}. As users' preferences evolve over time, capturing the temporal user behavior becomes essential. This is where SR steps in, attracting substantial research attention~\cite{hidasi2016gru4rec,wu2022surveyrec,gao2023surveyrec,tang2018caser,kang2018sasrec,chen2019behavior,schedl2018current,hansen2020contextual,jiang2016personalized,huang2018csan}. 

\begin{figure}[t]
    \centering
    \includegraphics[width=\columnwidth]{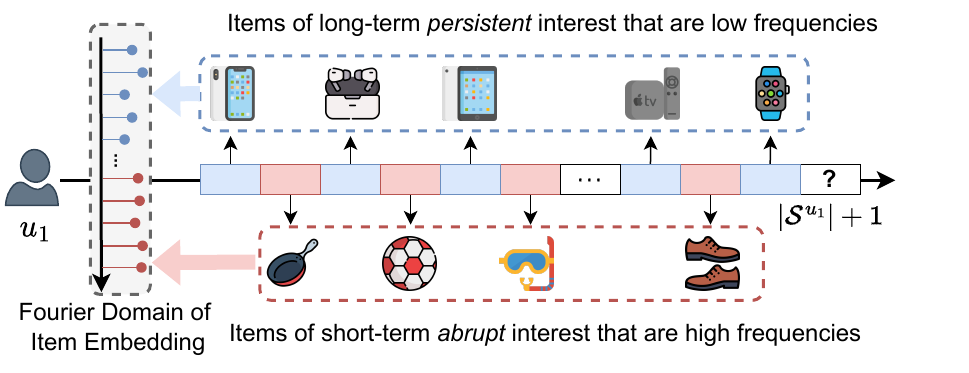}
    \caption{Illustration of high and low-frequency signals in SR. A user $u_1$'s long-term persisting interests and tastes constitute low frequencies in the Fourier domain of embedding, and abrupt short-term changes in $u_1$'s interests correspond to high frequencies.}
    \label{fig:example}
\end{figure}

\begin{table}[t]
    \centering
    \small
    \setlength{\tabcolsep}{1.4pt}
    \begin{tabular}{cccc}\toprule
        Method   & Inductive Bias & Self-Attention & High-pass Filter\\ \midrule
        SASRec   & \xmark & \cmark & \xmark \\
        BERT4Rec & \xmark & \cmark & \xmark \\
        FMLPRec & \cmark & \xmark & \xmark \\
        DuoRec   & \xmark & \cmark & \xmark \\
        \midrule
        BSARec   & \cmark & \cmark & \cmark \\
        \bottomrule
    \end{tabular}
    \caption{Comparison of existing Transformer-based methods that differ at three points: i) using inductive bias, ii) using self-attentions, and iii) using high-pass filters}
    \label{tab:intro}
\end{table}

With the increasing popularity of sequential recommendation (SR) systems, Transformer-based models, especially those utilizing self-attention~\cite{vaswani2017attention}, have emerged as dominant approaches for providing accurate and personalized recommendations to users~\cite{kang2018sasrec,sun2019bert4rec,li2020tisrec,wu2020sse,wu2020deja}. However, despite their successes in the SR, Transformer-based models possess inherent limitations that confine themselves to the learned self-attention matrix. The following two key limitations need to be addressed: i) First, the models may still suffer from suboptimal performance due to insufficient \emph{inductive bias} inherent in processing sequences with self-attention~\cite{dosovitskiy2020image}. While the self-attention mechanism captures long-range dependencies, it may not only adequately consider certain fine-grained sequential patterns but also be overfitted to training data, leading to potential weak generalization capabilities. As Table~\ref{tab:intro} shows, SASRec~\cite{kang2018sasrec}, BERT4Rec~\cite{sun2019bert4rec}, and DuoRec~\cite{qiu2022duorec} rely on training the self-attention layer and lack inductive bias\footnote{We mean by inductive bias a pre-determined attention structure that is not trained but injected by us when designing our model. Therefore, we call it as \emph{attentive inductive bias}.}.
ii) The second limitation pertains to the low-pass filtering nature of self-attention. By focusing on the entire range of data, self-attention may unintentionally smoothen out important and detailed patterns in embedding, resulting in the \emph{oversmoothing} problem. The oversmoothing issue poses a significant challenge in the SR domain, as it may hinder the ability of model to capture crucial temporal dynamics and provide accurate predictions. In Table~\ref{tab:intro}, most Transformer-based models are limited to low-pass filters. These models do not consider high-pass filters. Note that FMLPRec~\cite{zhou2022fmlprec} attempts to learn a filter, but it tends to gravitate towards the low-pass filter (cf. Fig.~\ref{fig:lpf} (b)). As shown in Fig.~\ref{fig:example}, the low-pass filter only captures the ongoing preferences of the user, i.e., an Apple fanatic, and it may be difficult to capture preferences based on new interests or trends (e.g., snorkel mask to buy for vacation). When recommending items for the next time ($|\mathcal{S}^{u_1}|+1$), it is undemanding to recommend long-term interests, but recommending short-term interests is a challenging task.

In this paper, we address these two limitations and present \textbf{B}eyond \textbf{S}elf-\textbf{A}ttention for Sequential \textbf{Rec}ommendation (BSARec), a novel model that uses inductive bias via Fourier transform with self-attention. By using the Fourier transform, BSARec gains access to the inductive bias of frequency information, enabling the capture of essential patterns and periodicity that may be overlooked by self-attention alone. This enhances inductive bias and has the potential to improve recommendation performance.

To tackle the oversmoothing issue, we introduce our own designed frequency rescaler to apply high-pass filters into BSARec's architecture. Our frequency rescaler can capture high-frequency behavioral patterns, such as interests driven by short-term trends, as well as low-frequency patterns, such as long-term interests, in a user's behavioral patterns (cf. Fig.~\ref{fig:example}). Additionally, our method provides a perspective to improve the performance of SR models and solve the problem of oversmoothing.

To evaluate the efficacy of BSARec, we conduct extensive experiments on 6 benchmark datasets. Our experimental results demonstrate that BSARec consistently outperforms 7 baseline methods regarding recommendation performance. Additionally, we conduct a series of experiments that underscore the necessity of our approach and verify its effectiveness in mitigating the oversmoothing problem, leading to improved recommendation accuracy and enhanced generalization capabilities. The contributions of this work are as follows:
\begin{itemize}
    \item We unveil the low-pass filtering nature of the self-attention of Transformer-based SR models, resulting in the problem of oversmoothing.
    \item We propose a novel model, \textbf{B}eyond \textbf{S}elf-\textbf{A}ttention for Sequential \textbf{Rec}ommendation (BSARec), that leverages the Fourier transform to balance between our inductive bias and self-attention. Further, we design the rescaler for high-pass filters to mitigate the oversmoothing issue.
    \item Extensive evaluation on 6 benchmark datasets demonstrates BSARec's outperformance over 7 baseline methods, validating its effectiveness in improving recommendation performance.
\end{itemize}

\section{Preliminaries}
\subsection{Problem Formulation}
The goal of SR is to predict the user’s next interaction with an item given their historical interaction sequences. Given a set of users $\mathcal{U}$ and items $\mathcal{V}$, we can sort the interacted items of each user $u \in \mathcal{U}$ chronologically in a sequence as $\mathcal{S}^{u}  = [v^{u}_1, v^{u}_2, \ldots v^{u}_{|\mathcal{S}^{u}|}]$, where $v^{u}_i$ denotes the $i$-th interacted item in the sequence. The aim is to recommend a Top-$k$ list of items as potential next items in a sequence. Formally, we predict $p(v^{u}_{|\mathcal{S}^{u}|+1}=v|\mathcal{S}^{u})$.

\subsection{Self-Attention for Sequential Recommendation}
The basic idea behind the self-attention mechanism is that elements within sequences are correlated but hold varying levels of significance concerning their positions in the sequence. Self-attention uses dot-products between items in the sequence to infer their correlations defined as: 
\begin{align}
    \mathbf{A} = \textrm{softmax}\left( \frac{\mathbf{Q}\mathbf{K}^{\mathtt{T}}}{\sqrt{d}} \right),
\end{align}where $\mathbf{Q}=\mathbf{E}_{\mathcal{S}^{u}}\mathbf{W}_Q$, $\mathbf{K}=\mathbf{E}_{\mathcal{S}^{u}}\mathbf{W}_K$, and $d$ is the scale factor. The scaled dot-product component learns the latent correlation between items. Other components in Transformer are utilized in SASRec, including the point-wise feed-forward network, residual connection, and layer normalization. Our method uses this self-attention matrix and adds an inductive bias to find the trade-off between the two methods.

\subsection{Discrete vs. Graph Fourier Transform} \label{sec:graph_fourier}
This subsection introduces the concept of the frequency domain and the Fourier transform, providing a cohesive foundation for the proposed method. 

The Discrete Fourier Transform (DFT) is a linchpin in digital signal processing (DSP), projecting a sequence of values into the frequency domain (or the Fourier domain). We typically use $\mathcal{F}: \mathbb{R}^N\rightarrow\mathbb{C}^{N}$ to denote DFT with the Inverse DFT (IDFT) $\mathcal{F}^{-1}: \mathbb{C}^{N} \rightarrow \mathbb{R}^N$. Applying $\mathcal{F}$ to a signal is equal to multiplying it from the left by a DFT matrix. The rows of this matrix consist of the Fourier basis $\bm{f}_j = [e^{2\pi i(j-1)\cdot0} \ldots e^{2\pi i(j-1)(N-1)}]^{\mathtt{T}}/\sqrt{N}\in\mathbb{R}^N$, where $i$ is the imaginary unit and $j$ denotes the $j$-th row. For the spectrum of $\bm{x}$, let it be represented as $\widebar{\bm{x}} = \mathcal{F}\bm{x}$. We can define $\widebar{\bm{x}}_{\text{lfc}}\in\mathbb{C}^c$ containing the $c$ lowest elements of $\widebar{\bm{x}}$, and $\widebar{\bm{x}}_{\text{hfc}} \in \mathbb{C}^{N-c} $ as the vector containing the remaining elements.
The low-frequency components (LFC) of the sequence signal $\bm{x}$ are defined as:
\begin{align}\label{eq:lfc}
\text{LFC}[\bm{x}] = \left[\bm{f}_1, \bm{f}_2, \ldots, \bm{f}_c\right]\widebar{\bm{x}}_{\text{lfc}} \in \mathbb{R}^N.
\end{align}
Conversely, the high-frequency components (HFC) are:
\begin{align}\label{eq:hfc}
\text{HFC}[\bm{x}] = \left[\bm{f}_{c+1}, \bm{f}_{c+2}, \ldots, \bm{f}_{N}\right] \widebar{\bm{x}}_{\text{hfc}} \in \mathbb{R}^N.
\end{align}
Note that we use real-valued DFT and multiplying with the Fourier bases in Eqs.~\ref{eq:lfc} and~\ref{eq:hfc} means IDFT. For more descriptions, interested readers should refer to Appendix.

The Graph Fourier Transform (GFT) can be considered as a generalization of DFT toward graphs. In other words, DFT is a special case of GFT, where a ring graph of $N$ nodes is used (see Fig.~\ref{fig:lpf} (a))~\cite{sandryhaila2014discrete}. In fact, DFT is a method to project a sequence of values onto the eigenspace of the Laplacian matrix of the ring graph  (which is the same as the Fourier domain).

The frequency concept can also be described with the ring graph. The number of neighboring nodes with different signs on their signals corresponds to the frequency. Therefore, low-frequency information means a series of signals over $N$ nodes whose signs do not change often. In the case of the SR in our work, where $N$ nodes mean $N$ item embeddings, such low-frequency information means a long-standing interest of a user (see Fig.~\ref{fig:example}).

\begin{figure}[t]
    \centering
    \subfigure[A ring graph]{\includegraphics[width=0.39\columnwidth]{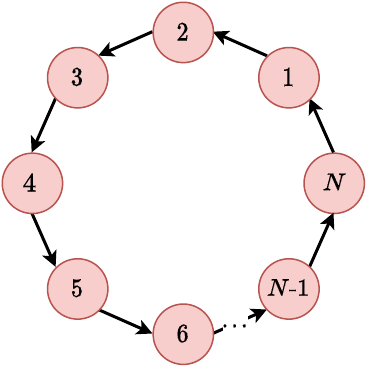}}
    \subfigure[{Spectral responses}]{\includegraphics[width=0.55\columnwidth]{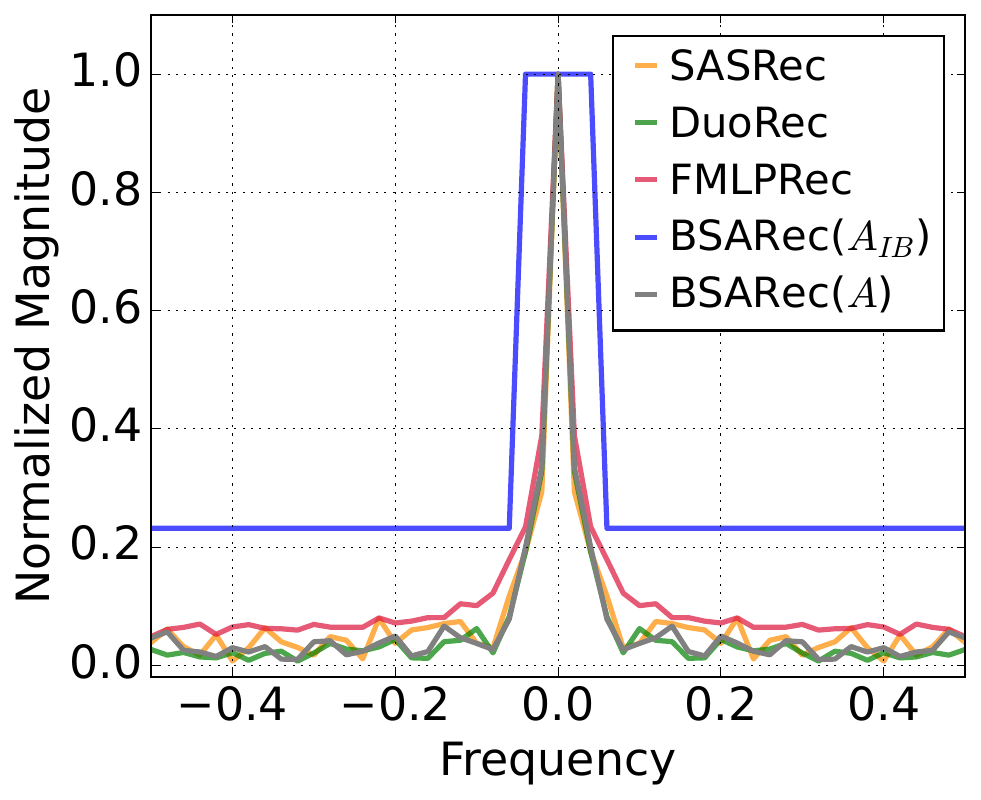}}
    \caption{(a) A ring graph with $N$ nodes, and (b) visualization of the filter of the self-attentions in LastFM.}
    \label{fig:lpf}
\end{figure}

\begin{figure}[t]
    \begin{subfigure}[{Singular value}]{\includegraphics[width=.49\columnwidth]{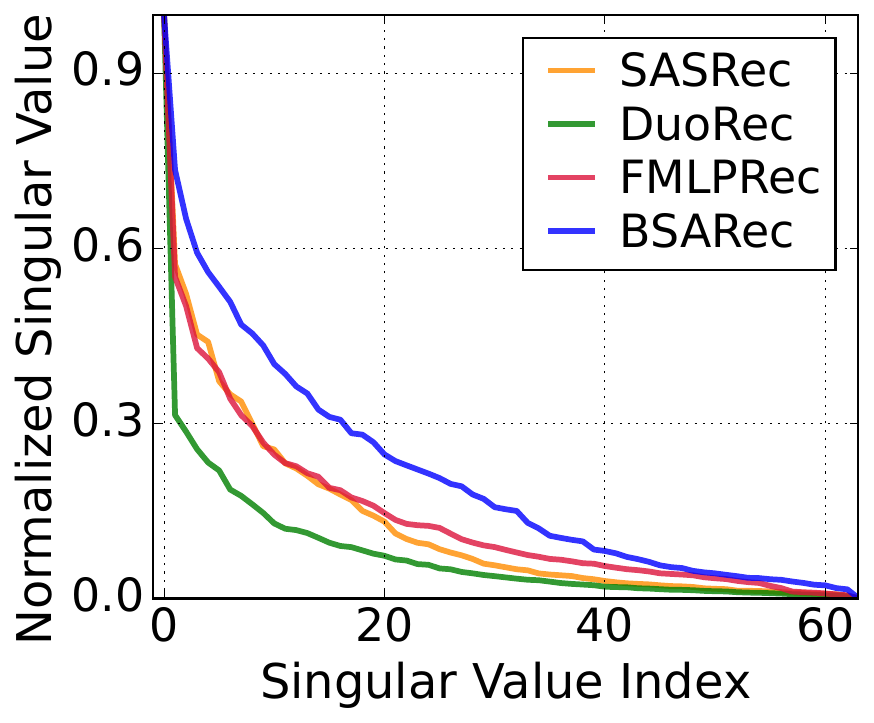}
        \label{fig:over_svd}}
    \end{subfigure}
    \hspace{-1em}
    \begin{subfigure}[Cosine similarity]{\includegraphics[width=.49\columnwidth]{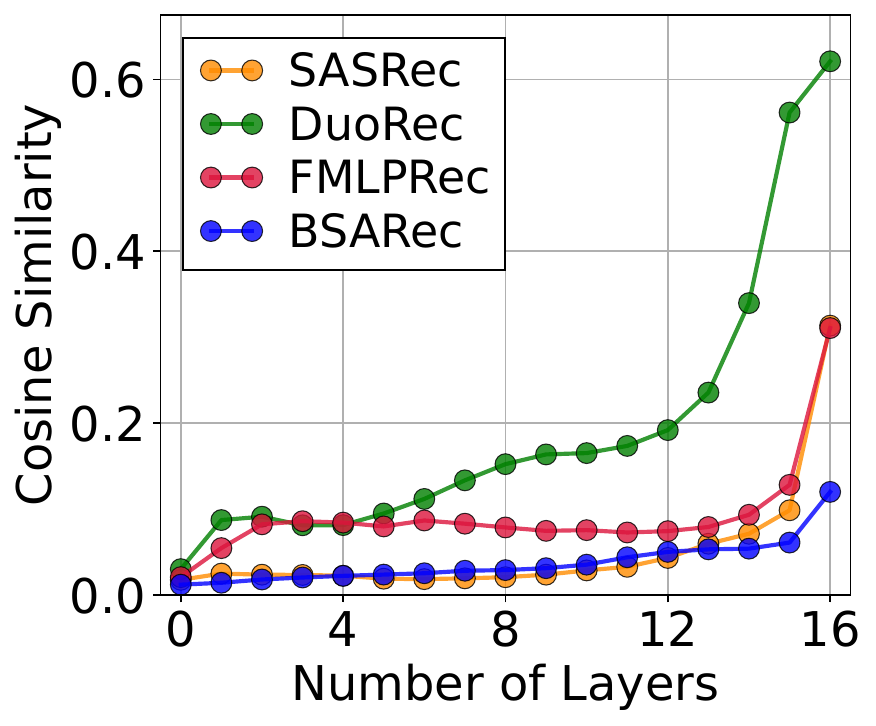}
        \label{fig:over_sim}}
    \end{subfigure}
    \caption{Visualization of oversmoothing in LastFM. The singular values and cosine similarity of user sequence output embedding.}
    \label{fig:over}
\end{figure}

\section{Motivation}
In this section, we show that self-attention in the spectral domain is a low-pass filter that continuously erases high-frequency information.
We visualize the spectrum of self-attention of the Transformer-based sequential model as shown in Fig.~\ref{fig:lpf} (b). It shows that the spectrum is concentrated in the low frequency region, and it reveals that self-attention is a low-pass filter. We further make theoretical justifications for the low-pass filter of self-attention.

\begin{theorem}[Self-Attention is a low-pass filter]\label{theorem:lpf}
Let $\mathbf{A} =  \textrm{softmax}(\mathbf{Q}\mathbf{K}^{\mathtt{T}}/\sqrt{d} )$. Then 
$\mathbf{A}$ inherently acts as a low-pass filter. For all $\bm{x}\in\mathbb{R}^N$, in other words, $\lim_{t\rightarrow \infty} ||\text{HFC}[\mathbf{A}^t(\bm{x})]||_2 / ||\text{LFC}[\mathbf{A}^t(\bm{x})]||_2=0$.
\end{theorem}

Theorem~\ref{theorem:lpf} is ensured by the Perron-Frobenius theorem~\cite{meyer2023matrix,he2021identifying}, revealing that the attention matrix is a low-pass filter independent of the input key and query matrices. A proof of Theorem~\ref{theorem:lpf} and the formal definition of the low-pass filter are provided in Appendix. If the self-attention matrix is applied successively, the final output loses all feature expressiveness as the number of layers increases to infinity. 

Therefore, the self-attention causes the oversmoothing problem that Tranformer-based SR models lose feature representation in deep layers (see Fig.~\ref{fig:over}). As can be seen from the empirical analysis of Fig.~\ref{fig:over}, as the number of layers of these models increases, the cosine similarity increases, and the singular value tends to decay rapidly~\footnote{This indicates that the largest singular value predominates and the other outliers are much smaller, and there is a potential risk of losing embedding rank.}~\cite{fan2023addressing}. This inevitably causes the model to fail to capture the user's detailed preferences, and performance degradation is a natural result.

We not only alleviate oversmoothing using a high-pass filter as motivation against this background, but also try to capture short-term preferences of user behavior patterns through inductive bias.

\section{Proposed Method}
Here, we introduce the overview of BSARec, the method behind our BSARec, and its relation to previous models.

\subsection{Embedding Layer}
Given a user's action sequence $\mathcal{S}^{u}$ and the maximum sequence length $N$, the sequence is first truncated by removing earliest item if $|\mathcal{S}^u|>N$ or padded with 0s to get a fixed length sequence $\bm{s}=(s_1,s_2,\ldots,s_N)$. With an item embedding matrix $\mathbf{M}\in \mathbb{R}^{|\mathcal{V}|\times D}$, we define the embedding representation of the sequence $\mathbf{E}^{u}$, where D is the latent dimension size and $\mathbf{E}^{u}_{i}=\mathbf{M}_{s_{i}}$. To make our model sensitive to the positions of items, we adopt positional embedding to inject additional positional information while maintaining the same embedding dimensions of the item embedding. A trainable positional embedding $\mathbf{P}\in \mathbb{R}^{N\times D}$ is added to the sequentially ordered item embedding matrix $\mathbf{E}^{u}$. Moreover, dropout and layer normalization are also implemented:
\begin{align}
    \mathbf{E}^{u} = \text{Dropout}(\text{LayerNorm}(\mathbf{E}^{u}+\mathbf{P} )).
\end{align}


\subsection{Beyond Self-Attention Encoder}
We develop item encoders by stacking beyond self-attention (BSA) blocks based on the embedding layer. It consists of 3 modules (see Fig.~\ref{fig:overall}): BSA layer, attentive inductive bias with frequency rescaler, and feed forward network.
\paragraph{Beyond Self-Attention Layer}
Let $\widetilde{\mathbf{A}}^{\ell}$ be a beyond self-attention (BSA), $\mathbf{A}_{\textrm{IB}}^{\ell}$ be a rescaled filter matrix for the $l$-th layer, and $\mathbf{X}^{\ell}$ is the input for the $l$-th layer. When $l=0$, we set $\mathbf{X}^0 = \mathbf{E}^{u}$. We use the following BSA layer:
\begin{align}\label{eq:bsa}
\mathbf{S}^{\ell} = \widetilde{\mathbf{A}}^{\ell}\mathbf{X}^{\ell} = \alpha \mathbf{A}_{\textrm{IB}}^{\ell}\mathbf{X}^{\ell} + (1-\alpha) \mathbf{A}^{\ell}\mathbf{X}^{\ell},
\end{align}where the first term corresponds to DSP, where the discrete Fourier transform is utilized, $\alpha \leq 1$ is a coefficient to (de-)emphasize the inductive bias. Therefore, our main design point is to trade off between the verified inductive bias and the trainable self-attention.

For the multi-head version used in BSARec, the multi-head self-attention (MSA) is defined as:
\begin{align}
    \widehat{\mathbf{X}}^{\ell} = \textrm{MSA}(\mathbf{X}^{\ell})=[\mathbf{S}^1, \mathbf{S}^2, \ldots \mathbf{S}^h]\mathbf{W}^{O},
\end{align}
where $h$ is the number of heads and the projection matrix $\mathbf{W}^{O} \in \mathbb{R}^{D \times D}$ is the learnable parameter.

\begin{figure}
    \centering
   \includegraphics[width=1.0\columnwidth]{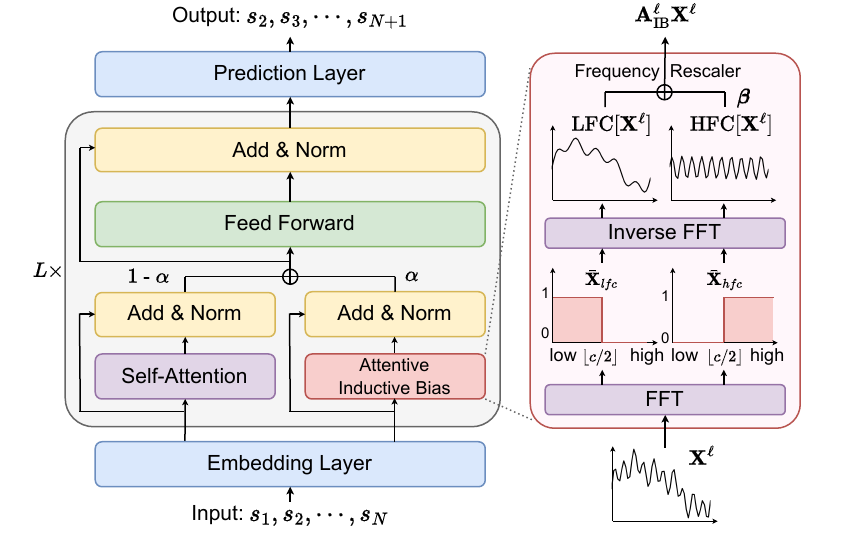}
    \caption{Architecture of our proposed BSARec. We propose a BSA encoder that uses both an inductive bias with a frequency rescaler and original self-attention.}
    \label{fig:overall}
\end{figure}

\subsubsection{Attentive Inductive Bias with Frequency Rescaler}
We propose a filter that injects the attentive inductive bias and, at the same time, adjusts the scale of the frequency by dividing it into low and high frequency components:
\begin{align}
\mathbf{A}_{\textrm{IB}}^{\ell} \mathbf{X}^{\ell} &= \text{LFC}[{\mathbf{X}^{\ell}}]+ \bm{\beta} \text{HFC}[\mathbf{X}^{\ell}],\label{ib}
\end{align}
where $\bm{\beta}$ is a trainable parameter to scale the high-pass filter. In particular, $\bm{\beta}$ can be either a vector with $D$ dimension or a scalar parameter.

\paragraph{Meaning of our Attentive Inductive Bias}We note that DFT is used in Eq.~\eqref{ib}, which assumes the ring graph in Fig.~\ref{fig:lpf} (a) --- in the perspective of self-attention, this inductive bias says that an item to purchase is influenced by its previous item. This attentive inductive bias does not need to be trained since we know that it presents universally for SR.

However, we do not stop at utilizing the inductive bias in a na\"ive way but extract its low and high-frequency information to learn how to optimally mix them in Eq.~\eqref{ib}. To be more specific, suppose a ring graph of $N$ item embeddings. $\text{LFC}[\cdot]$ on them extracts their common signals that do not greatly change following the ring graph topology whereas $\text{HFC}[\cdot]$  extracts locally fluctuating signals (see Fig.~\ref{tab:intro}). By selectively utilizing the high-pass information, we can prevent the oversmoothing problem (see Fig.~\ref{fig:over}). If relying on $\text{LFC}[\cdot]$ only, we cannot prevent the oversmoothing problem.

In addition, we also learn the self-attention matrix $\mathbf{A}^{\ell}$ in Eq.~\eqref{eq:bsa} and combine it with our attentive inductive bias $\mathbf{A}_{\textrm{IB}}^{\ell}$. By separating $\mathbf{A}^{\ell}$ from $\widetilde{\mathbf{A}}^{\ell}$, the self-attention mechanism focuses on capturing non-obvious attentions in $\mathbf{A}^{\ell}$.

\paragraph{Point-wise Feed-Forward Network and Layer Outputs}
The multi-head attention function is primarily based on linear projection. A point-wise feed-forward network is applied to import nonlinearity to the self-attention block. The process is defined as follows:
\begin{align}
    \widetilde{\mathbf{X}}^{\ell} = (\textrm{GELU}(\widehat{\mathbf{X}}^{\ell} \mathbf{W}^{\ell}_1 +\mathbf{b}^{\ell}_1) )\mathbf{W}^{\ell}_2 + \mathbf{b}^{\ell}_2,
\end{align}
where $\mathbf{W}^{\ell}_1, \mathbf{W}^{\ell}_2 \in \mathbb{R}^{D\times D}$ and  $\mathbf{b}^{\ell}_1, \mathbf{b}^{\ell}_2 \in \mathbb{R}^{D\times D}$ are learnable parameters. The dropout layer, residual connection, and layer normalization operations are applied as follows:
\begin{align}
    \mathbf{X}^{{\ell}+1} = \textrm{LayerNorm}(\mathbf{X}^{\ell} + \widehat{\mathbf{X}}^{\ell} + \textrm{Dropout}(\widetilde{\mathbf{X}}^{\ell})).
\end{align}

\subsection{Prediction Layer and Training}
In the final layer of BSARec, we calculate the user's preference score for the item $i$ derived from user's historical interactions. This score is given by:
\begin{align}
    \hat{y}_i = p(v^{u}_{|\mathcal{S}^{u}|+1}=v|\mathcal{S}^{u}) =\mathbf{e}_v^{\mathtt{T}}{\mathbf{X}^{L}_{|\mathcal{S}^{u}|}},
\end{align} where $\mathbf{e}_v$ is the representation of item $v$ from $\mathbf{M}$, and $\mathbf{X}^{L}_{|\mathcal{S}^{u}|}$ is the output of the $L$-layer blocks at step $|\mathcal{S}^{u}|$. This dot product computes the similarity between these two vectors to give us the preference score $\hat{y}_i$.

The cross-entropy (CE) loss function is usually used in SR since the next item prediction task is treated as a classification task over the whole item set~\cite{zhang2019feature,qiu2022duorec,du2023fearec}. We adopt the CE loss to optimize the model parameter as:
\begin{align}
    \mathcal{L} = -log\frac{\exp(\hat{y}_g)}{\sum_{i\in|\mathcal{V}|}{\exp(\hat{y}_i)}},
\end{align} where $g\in|\mathcal{V}|$ is the ground truth item.

\begin{table*}[t]
    \centering
    \footnotesize
    \resizebox{.95\textwidth}{!}{%
    \begin{tabular}{ll ccccccccc}
    \specialrule{1pt}{1pt}{2pt}
    Datasets                & Metric    & Caser & GRU4Rec & SASRec    & BERT4Rec  & FMLPRec  & DuoRec    & FEARec    & BSARec  & Improv.    \\
    \specialrule{1pt}{1pt}{2pt}

\multirow{6}{*}{Beauty} 
  & HR@5    & 0.0125 & 0.0169 & 0.0340 & 0.0469 & 0.0346 & \underline{0.0707} & 0.0706 & \textbf{0.0736} & 4.10\%\\ 
  & HR@10   & 0.0225 & 0.0304 & 0.0531 & 0.0705 & 0.0559 & 0.0965 & \underline{0.0982} & \textbf{0.1008} & 2.65\%\\ 
  & HR@20   & 0.0403 & 0.0527 & 0.0823 & 0.1073 & 0.0869 & 0.1313 & \underline{0.1352} & \textbf{0.1373} & 1.55\%\\ 
  & NDCG@5  & 0.0076 & 0.0104 & 0.0221 & 0.0311 & 0.0222 & 0.0501 & \underline{0.0512} & \textbf{0.0523} & 2.15\%\\ 
  & NDCG@10 & 0.0108 & 0.0147 & 0.0283 & 0.0387 & 0.0291 & 0.0584 & \underline{0.0601} & \textbf{0.0611} & 1.66\%\\ 
  & NDCG@20 & 0.0153 & 0.0203 & 0.0356 & 0.0480 & 0.0369 & 0.0671 & \underline{0.0694} & \textbf{0.0703} & 1.30\%\\ 
\specialrule{0.5pt}{1pt}{2pt} 
\multirow{6}{*}{Sports} 
  & HR@5    & 0.0091 & 0.0118 & 0.0188 & 0.0275 & 0.0220 & 0.0396 & \underline{0.0411} & \textbf{0.0426} & 3.65\%  \\ 
  & HR@10   & 0.0163 & 0.0187 & 0.0298 & 0.0428 & 0.0336 & 0.0569 & \underline{0.0589} & \textbf{0.0612} & 3.90\%  \\ 
  & HR@20   & 0.0260 & 0.0303 & 0.0459 & 0.0649 & 0.0525 & 0.0791 & \underline{0.0836} & \textbf{0.0858} & 2.63\%  \\ 
  & NDCG@5  & 0.0056 & 0.0079 & 0.0124 & 0.0180 & 0.0146 & 0.0276 & \underline{0.0286} & \textbf{0.0300} & 4.90\%  \\ 
  & NDCG@10 & 0.0080 & 0.0101 & 0.0159 & 0.0229 & 0.0183 & 0.0331 & \underline{0.0343} & \textbf{0.0360} & 4.96\%  \\ 
  & NDCG@20 & 0.0104 & 0.0131 & 0.0200 & 0.0284 & 0.0231 & 0.0387 & \underline{0.0405} & \textbf{0.0422} & 4.20\%  \\ 
\specialrule{0.5pt}{1pt}{2pt} 
\multirow{6}{*}{Toys} 
  & HR@5    & 0.0095 & 0.0121 & 0.0440 & 0.0412 & 0.0432 & 0.0770 & \underline{0.0783} & \textbf{0.0805} & 2.81\%\\ 
  & HR@10   & 0.0161 & 0.0211 & 0.0652 & 0.0635 & 0.0671 & 0.1034 & \underline{0.1054} & \textbf{0.1081} & 2.56\%\\ 
  & HR@20   & 0.0268 & 0.0348 & 0.0929 & 0.0939 & 0.0974 & 0.1369 & \underline{0.1397} & \textbf{0.1435} & 2.72\%\\ 
  & NDCG@5  & 0.0058 & 0.0077 & 0.0297 & 0.0282 & 0.0288 & 0.0568 & \underline{0.0574} & \textbf{0.0589} & 2.61\%\\ 
  & NDCG@10 & 0.0079 & 0.0106 & 0.0366 & 0.0353 & 0.0365 & 0.0653 & \underline{0.0661} & \textbf{0.0679} & 2.72\%\\ 
  & NDCG@20 & 0.0106 & 0.0140 & 0.0435 & 0.0430 & 0.0441 & 0.0737 & \underline{0.0747} & \textbf{0.0768} & 2.81\%\\ 
\specialrule{0.5pt}{1pt}{2pt} 
\multirow{6}{*}{Yelp} 
  & HR@5    & 0.0117 & 0.0130 & 0.0149 & 0.0256 & 0.0159 & \underline{0.0271} & 0.0262 & \textbf{0.0275} & 1.48\%  \\ 
  & HR@10   & 0.0197 & 0.0221 & 0.0249 & 0.0433 & 0.0287 & \underline{0.0442} & 0.0437 & \textbf{0.0465} & 5.20\%  \\ 
  & HR@20   & 0.0337 & 0.0383 & 0.0424 & \underline{0.0717} & 0.0490 & \underline{0.0717} & 0.0691 & \textbf{0.0746} & 4.04\%  \\ 
  & NDCG@5  & 0.0070 & 0.0080 & 0.0091 & 0.0159 & 0.0100 & \textbf{0.0170}    & 0.0165 & \textbf{0.0170} & 0.00\%  \\ 
  & NDCG@10 & 0.0096 & 0.0109 & 0.0123 & 0.0216 & 0.0142 & \underline{0.0225} & 0.0221 & \textbf{0.0231} & 2.67\%  \\ 
  & NDCG@20 & 0.0131 & 0.0150 & 0.0167 & 0.0287 & 0.0192 & \underline{0.0294} & 0.0285 & \textbf{0.0302} & 2.72\%  \\ 
\specialrule{0.5pt}{1pt}{2pt} 
\multirow{6}{*}{LastFM} 
  & HR@5    & 0.0303 & 0.0312 & 0.0413 & 0.0294 & 0.0367 & \underline{0.0431} & \underline{0.0431} & \textbf{0.0523} & 21.35\%  \\ 
  & HR@10   & 0.0431 & 0.0404 & \underline{0.0633} & 0.0459 & 0.0560 & 0.0624 & 0.0587 &             \textbf{0.0807} & 27.49\%  \\ 
  & HR@20   & 0.0642 & 0.0541 & 0.0927 & 0.0596 & 0.0826 & \underline{0.0963} & 0.0826 &             \textbf{0.1174} & 21.91\%  \\ 
  & NDCG@5  & 0.0227 & 0.0217 & 0.0284 & 0.0198 & 0.0243 & 0.0300 & \underline{0.0304} &             \textbf{0.0344} & 13.16\%  \\ 
  & NDCG@10 & 0.0268 & 0.0245 & 0.0355 & 0.0252 & 0.0306 & \underline{0.0361} & 0.0354 &             \textbf{0.0435} & 20.50\%  \\ 
  & NDCG@20 & 0.0321 & 0.0280 & 0.0429 & 0.0286 & 0.0372 & \underline{0.0446} & 0.0414 &             \textbf{0.0526} & 17.94\%  \\ 
\specialrule{0.5pt}{1pt}{2pt} 
\multirow{6}{*}{ML-1M} 
  & HR@5    & 0.0927 & 0.1005 & 0.1374 & 0.1512 & 0.1316 & \underline{0.1838} & 0.1834 & \textbf{0.1944} & 5.77\%  \\ 
  & HR@10   & 0.1556 & 0.1657 & 0.2137 & 0.2346 & 0.2065 & 0.2704 & \underline{0.2705} & \textbf{0.2757} & 1.92\%  \\ 
  & HR@20   & 0.2488 & 0.2664 & 0.3245 & 0.3440 & 0.3137 & \underline{0.3738} & 0.3714 & \textbf{0.3884} & 3.91\%  \\ 
  & NDCG@5  & 0.0592 & 0.0619 & 0.0873 & 0.1021 & 0.0846 & \underline{0.1252} & 0.1236 & \textbf{0.1306} & 4.31\%  \\ 
  & NDCG@10 & 0.0795 & 0.0828 & 0.1116 & 0.1289 & 0.1087 & \underline{0.1530} & 0.1516 & \textbf{0.1568} & 2.48\%  \\ 
  & NDCG@20 & 0.1028 & 0.1081 & 0.1395 & 0.1564 & 0.1356 & \underline{0.1790} & 0.1771 & \textbf{0.1851} & 3.41\%  \\ 
        \specialrule{1pt}{1pt}{1pt}
    \end{tabular}%
    }
    \caption{Performance comparison of different methods on 6 datasets. The best results are in  boldface and the second-best results are underlined.`Improv.' indicates the relative improvement against the best baseline performance.}
    \label{tab:main_exp}
\end{table*}

\subsection{Relation to Previous Models}
Several Transformer-based SR models can be a special case of BSARec, and the comparison with existing models is as follows:
i) When $\alpha$ is 0 in BSARec, our model is reduced to SASRec. This is because pure self-attention is used as it is. However, one difference is that their loss functions are different. BSARec uses the CE loss, while SASRec uses the BCE loss.  Even in the case of DuoRec, which extends SASRec with contrastive learning, it can be a BSARec with $\alpha=0$ except for contrastive learning.
ii) In the case of the FMLPRec, it uses DFT only without self-attention. Nevertheless, the most significant difference is that the filter matrix itself in FMLPRec is a learnable matrix. Because of this, FMLPRec's filter is inevitably learned as a low-pass filter, while BSARec uses a filter rescaler to simultaneously use a high-pass filter.
iii) Similar to BSARec, FEARec separates low-frequency and high-frequency information in the frequency domain. However, FEARec allows the frequency domain to be learned separately before entering its Transformer's encoder. BSARec adaptively uses low and high-frequency information by using a frequency rescaler in a step to inject an inductive bias. FEARec is designed with a complex model structure using contrastive learning and frequency normalization. However, our model shows better performance with a much simpler architecture.

\section{Experiments}
\subsection{Experimental Setup}
\paragraph{Datasets} We evaluate our model on 6 SR datasets where the sparsity and domain varies: i,ii,iii) Amazon Beauty, Sports, Toys~\cite{mcauley2015image}, iv) Yelp, v) ML-1M~\cite{harper2015movielens}, and vi) LastFM. We followed the data pre-processing procedure from \citet{zhou2020s3, zhou2022fmlprec}, where all reviews and ratings are regarded as implicit feedback. 
The detailed dataset statistics are presented in Appendix.

\paragraph{Baselines} To verify the effectiveness of our model, we compare our method with well-known SR baselines with three categories:
\begin{itemize}
    \item RNN or CNN-based sequential models: GRU4Rec~\cite{hidasi2016gru4rec} and Caser~\cite{tang2018caser}.
    \item Transformer-based sequential models: SASRec~\cite{kang2018sasrec}, BERT4Rec~\cite{sun2019bert4rec}, and FMLPRec~\cite{zhou2022fmlprec}.
    \item Transformer-based sequential models with contrsastive learning: DuoRec~\cite{qiu2022duorec} and FEARec~\cite{du2023fearec}.
\end{itemize}

\paragraph{Implementation Details}
Our method is implemented in PyTorch on an NVIDIA RTX 3090 with 16 GB memory.
We search the best hyperparameters for baselines based on their recommended hyperparameters. We conduct experiments under the following hyperparameters: the coefficient $\alpha$ is in $\{0.1, 0.3, 0.5, 0.7, 0.9\}$, and $c$ is chosen from $\{1, 3, 5, 7, 9\}$. The number of BSA blocks $L$ is set to 2, and the number of heads in Transformer $h$ is in $\{1, 2, 4\}$. The dimension of $D$ is set to 64, and the maximum sequence length $N$ is set to 50. For training, the Adam optimizer is optimized with a learning rate in \{\num{5e-4}, \num{1e-3}\}, and the batch size is set to 256. The best hyperparameters are in Appendix for reproducibility.

\paragraph{Metrics}To measure the recommendation accuracy, we commonly use widely used Top-$k$ metrics, HR@$k$ (Hit Rate) and NDCG@$k$ (Normalized Discounted Cumulative Gain) to evaluate the recommended list, where $k$ is set to 5, 10, and 20. To ensure a fair and comprehensive comparison, we analyze the ranking results across the full item set without negative sampling~\cite{krichene2020sampled}.

\subsection{Experimental Results}
Table~\ref{tab:main_exp} presents the detailed recommendation performance. Overall, our proposed method, BSARec, clearly marks the best accuracy. First, compared to existing RNN-based and CNN-based methods, Transformer-based methods show better performance in modeling interaction sequences in SR.
Second, in Transformer-based methods, BERT4Rec and FMLPRec models outperform SASRec. 
In particular, FMLPRec redesigned the self-attention of the existing Transformer only with MLP, but it still does not perform well in all datasets.
Third, there is no doubt that models using contrastive learning show higher results than models that do not. DuoRec and FEARec significantly outperform SASRec, BERT4Rec, and FMLPRec.

Surprisingly, however, BSARec records the best performance across all datasets and all metrics. The most surprising thing is that it can show better performance than DuoRec and FEARec without using contrastive learning. In LastFM, BSARec shows a performance improvement of 27.49\% based on HR@10. Thus, our model leaves a message that it can show good performance without going to complex model design by adding contrastive learning.

\begin{table}[t]
    \centering
    \setlength{\tabcolsep}{1.5pt}
    \begin{tabular}{l cccccc}\toprule
    \multirow{2}{*}{Methods}    && \multicolumn{2}{c}{Beauty}     && \multicolumn{2}{c}{Toys}  \\ \cmidrule(lr){3-4}\cmidrule(lr){6-7}
                                && HR@20 &  NDCG@20   &&  HR@20 &  NDCG@20  \\ 
    \midrule
    BSARec                      && \textbf{0.1373} & \textbf{0.0703} && \textbf{0.1435} & \textbf{0.0768} \\
    \midrule
    Only $\mathbf{A}$           && 0.1265 & 0.0657 && 0.1320 & 0.0720 \\
    Only $\mathbf{A}_{\textrm{IB}}$    && 0.1338 & 0.0677 && 0.1402 & 0.0744 \\
    Scalar $\beta$                 && 0.1333 & 0.0685 && \textbf{0.1435} & 0.0756\\
    \bottomrule
    \end{tabular}
    \caption{Ablation studies on $\widetilde{\mathbf{A}}$ and $\bm{\beta}$. More results in other datasets are in Appendix.}
    \label{tab:ablation}
\end{table}

\subsection{Ablation, Sensitivity, and Additional Studies}
\paragraph{Ablation Studies}
As ablation study models, we define the following models: i) the first ablation model has only the self-attention term i.e., $\mathbf{A}$, ii) the second ablation model has only the attentive inductive bias term, i.e., $\mathbf{A}_{\textrm{IB}}$ in Eq.~\ref{eq:bsa}, and iii) the third ablation model uses $\beta$ as a single parameter. For Beauty and Toys, the ablation study model with only $\mathbf{A}_{\textrm{IB}}$ outperforms the case with only $\mathbf{A}$ (e.g., HR@20 in Beauty by $\mathbf{A}_{\textrm{IB}}$ of 0.1338 versus 0.1265 by $\mathbf{A}$). However, BSARec, which utilizes them all, outperforms them. This shows that both are required to achieve the best accuracy.

\paragraph{Sensitivity to $\alpha$} Fig.~\ref{fig:sen_alpha} shows the NDCG@20 and HR@20 by varying the $\alpha$. For Beauty, we find our BSARec, a larger value of $\alpha$ is preferred. For ML-1M, with $\alpha=0.3$, we can achieve the best accuracy. The trade-off between the self-attention matrix and the inductive bias differs for each dataset from these results.
\paragraph{Sensitivity to $c$}

Fig.~\ref{fig:sen_c} shows the NDCG@20 and HR@20 by varying the $c$. For Beauty, the best accuracy is achieved when $c$ is 5. For ML-1M, the larger the value of $c$, the better performance is reached.

\paragraph{Visualization of Learned $\bm{\beta}$} 
In Fig.~\ref{fig:beta} (a), we show learned $\bm{\beta}$ at each layer for all datasets. We can see that a higher weight is learned in the first layer than in the second layer, which confirms that putting more weight on high-frequency in the first layer is effective. In particular, LastFM and Beauty show higher $\bm{\beta}$ weights than other datasets.

\begin{figure}[t]
    \begin{subfigure}[{Beauty}]{\includegraphics[width=.494\columnwidth]{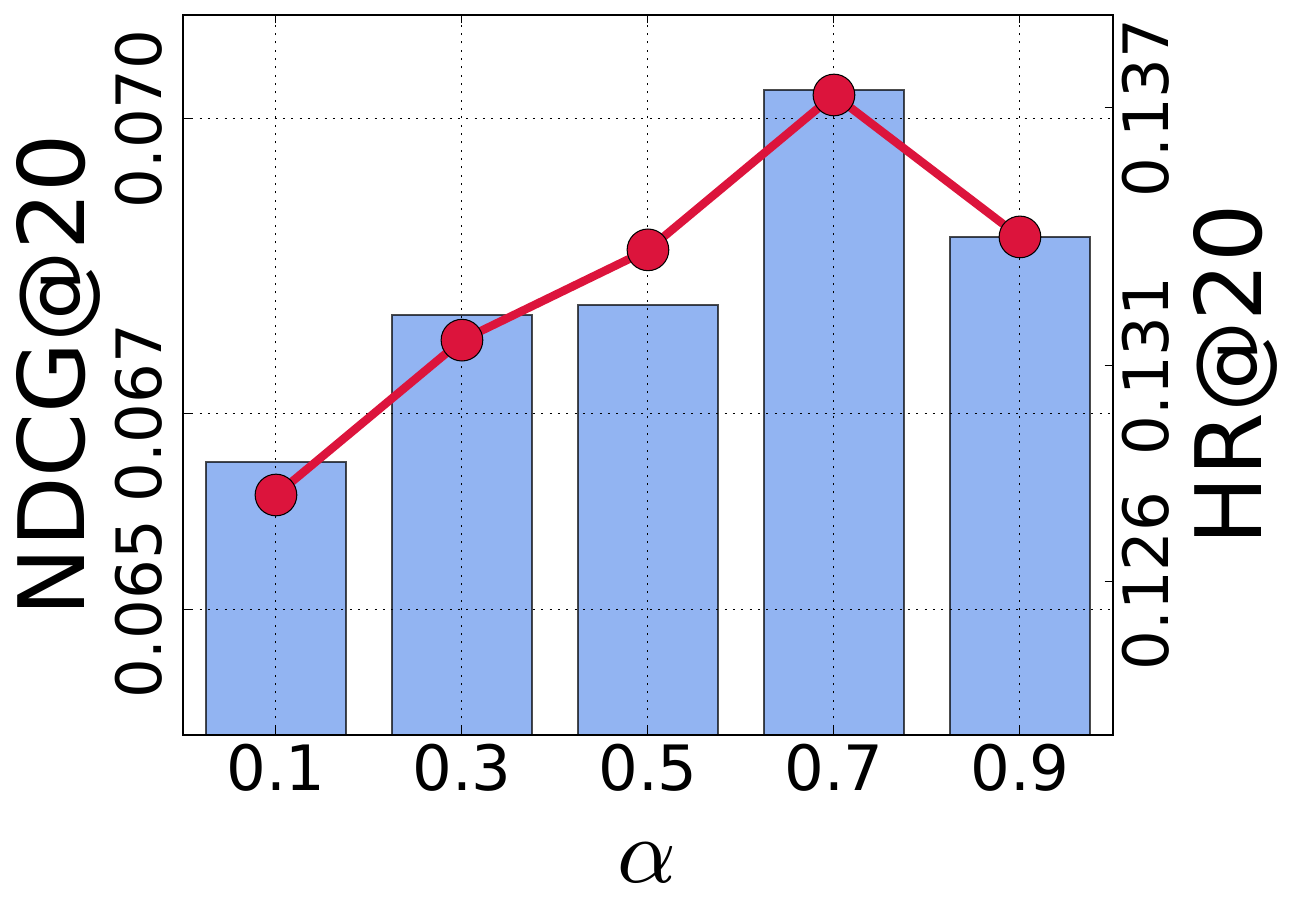}
        \label{fig:sen_alpha_Beauty}}
    \end{subfigure}
    \hspace{-1em}
    \begin{subfigure}[ML-1M]{\includegraphics[width=.494\columnwidth]{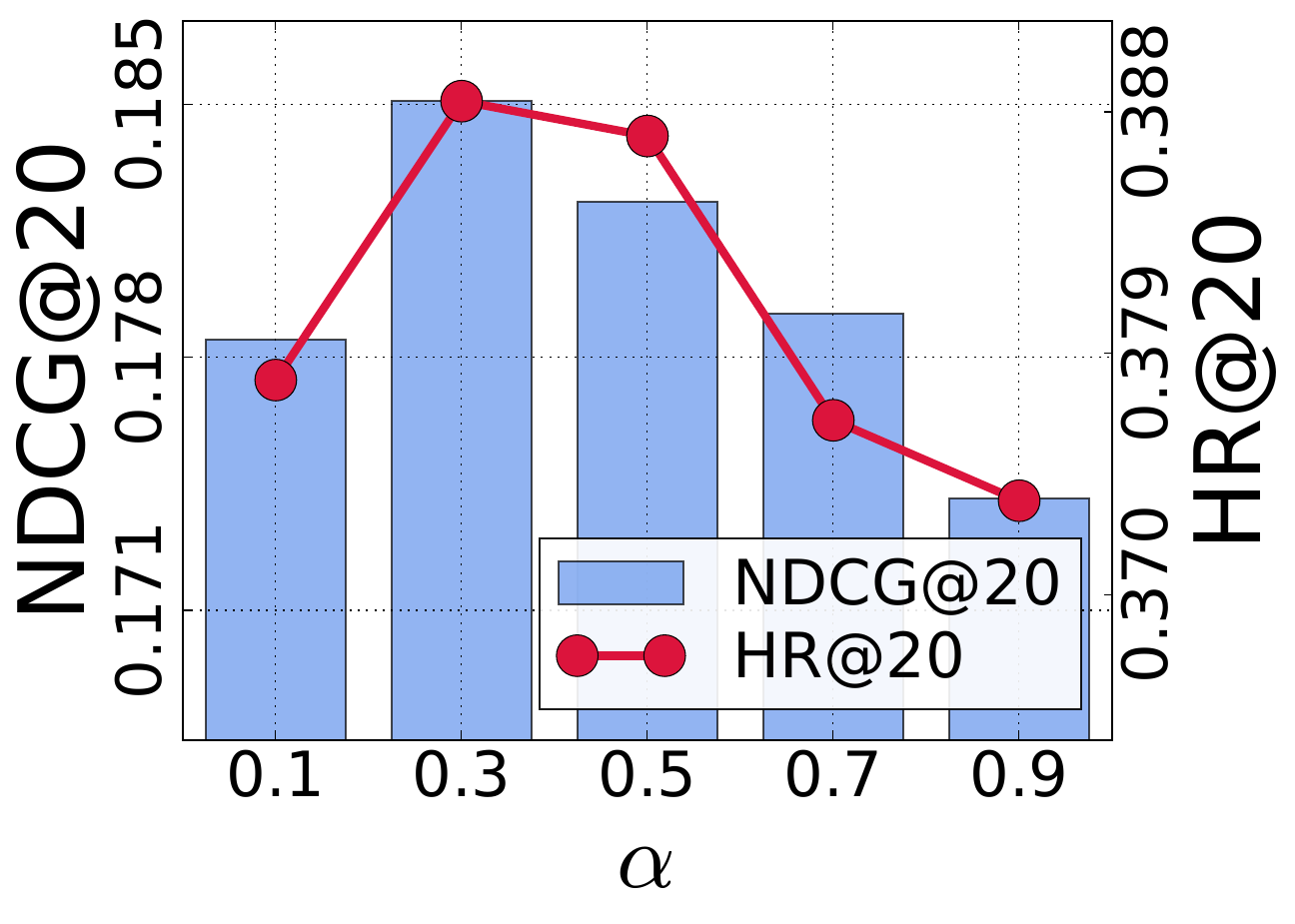}
        \label{fig:sen_alpha_LastFM}}
    \end{subfigure}
    \caption{Sensitivity to $\alpha$. More results in other datasets are in Appendix.}
    \label{fig:sen_alpha}
\end{figure}

\begin{figure}[t]
    \begin{subfigure}[Beauty]{\includegraphics[width=.494\columnwidth]{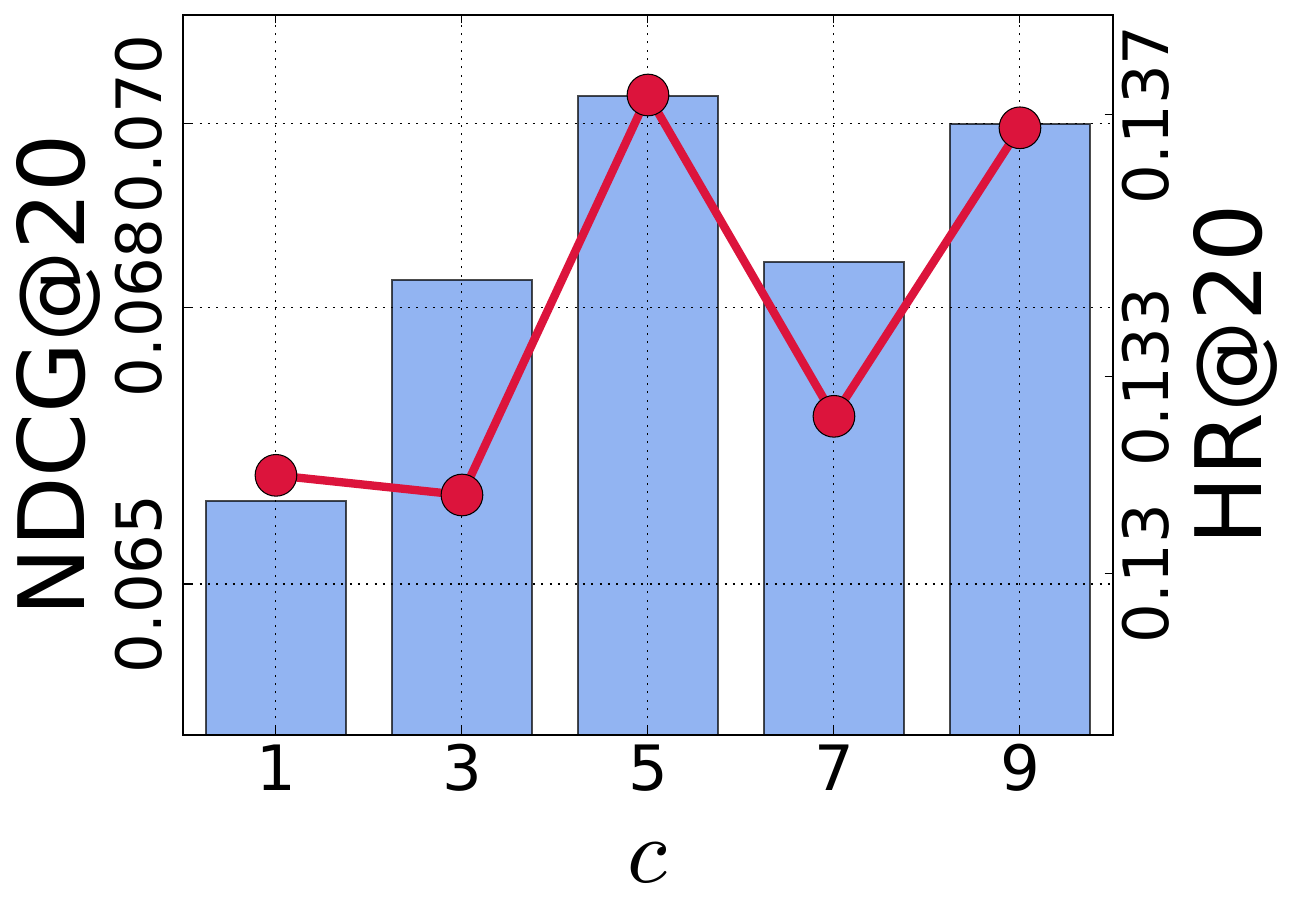}
        \label{fig:sen_k_Beauty}}
    \end{subfigure}
    \hspace{-1em}
    \begin{subfigure}[ML-1M]{\includegraphics[width=.494\columnwidth]{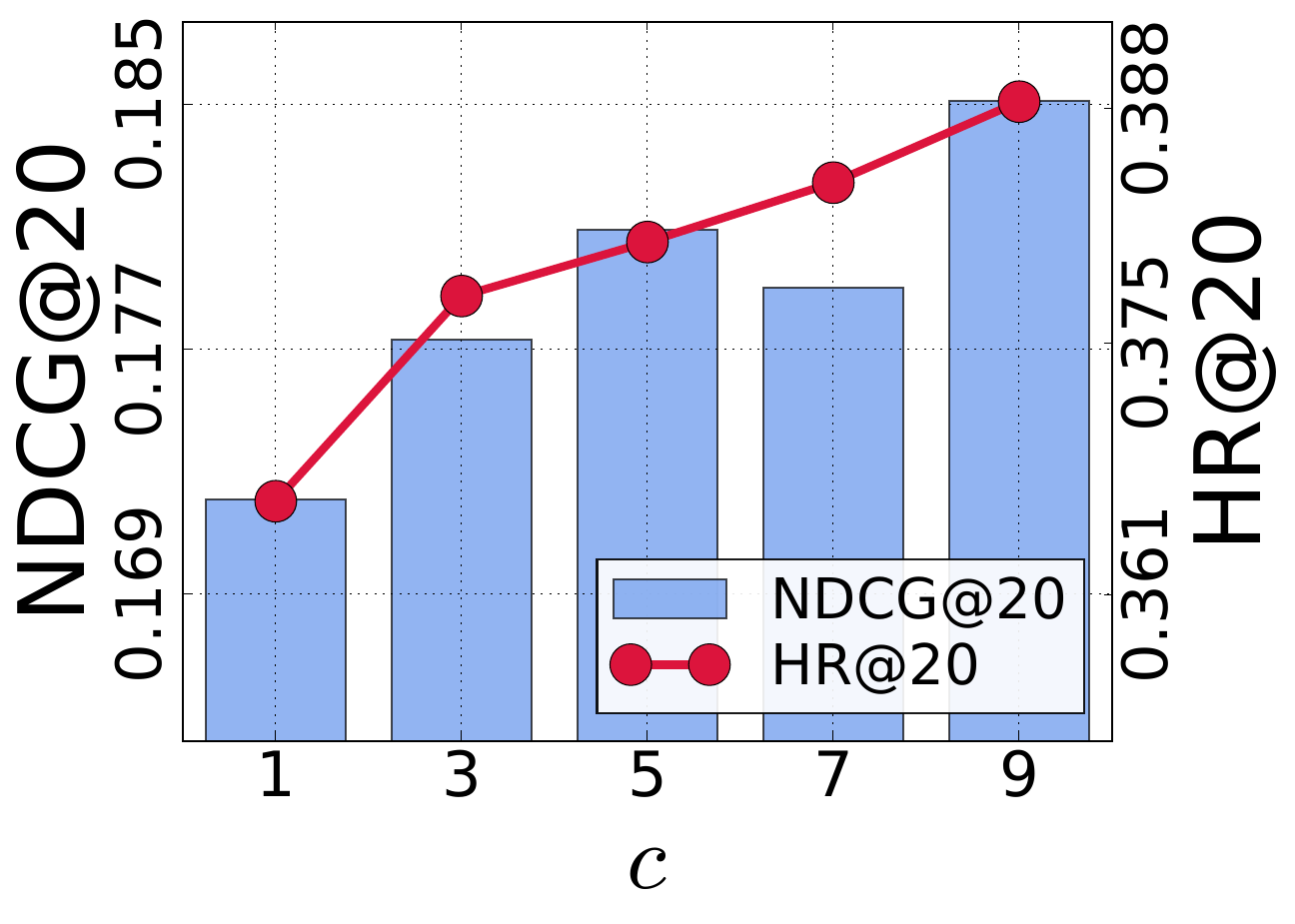}
        \label{fig:sen_k_LastFM}}
    \end{subfigure}
    \caption{Sensitivity to $c$. More results in other datasets are in Appendix.}
    \label{fig:sen_c}
\end{figure}

\begin{figure}[t]
    \centering
    \subfigure[{Visualization of learned $\bm{\beta}$}]{\includegraphics[width=0.495\columnwidth]{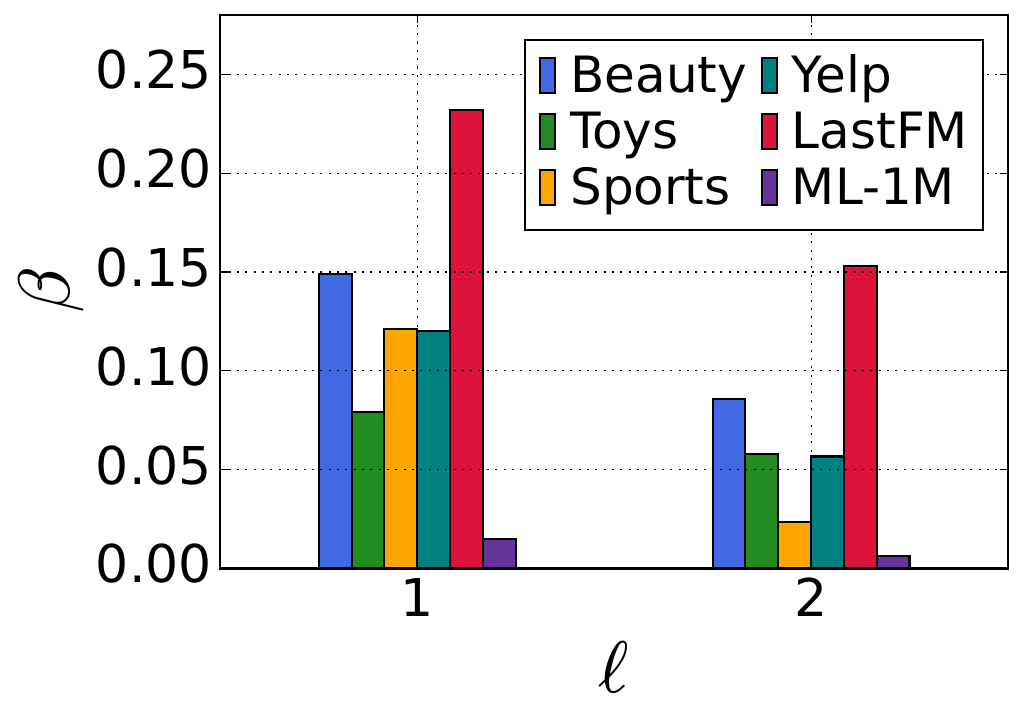}}
    \subfigure[Case study]{\includegraphics[width=0.495\columnwidth]{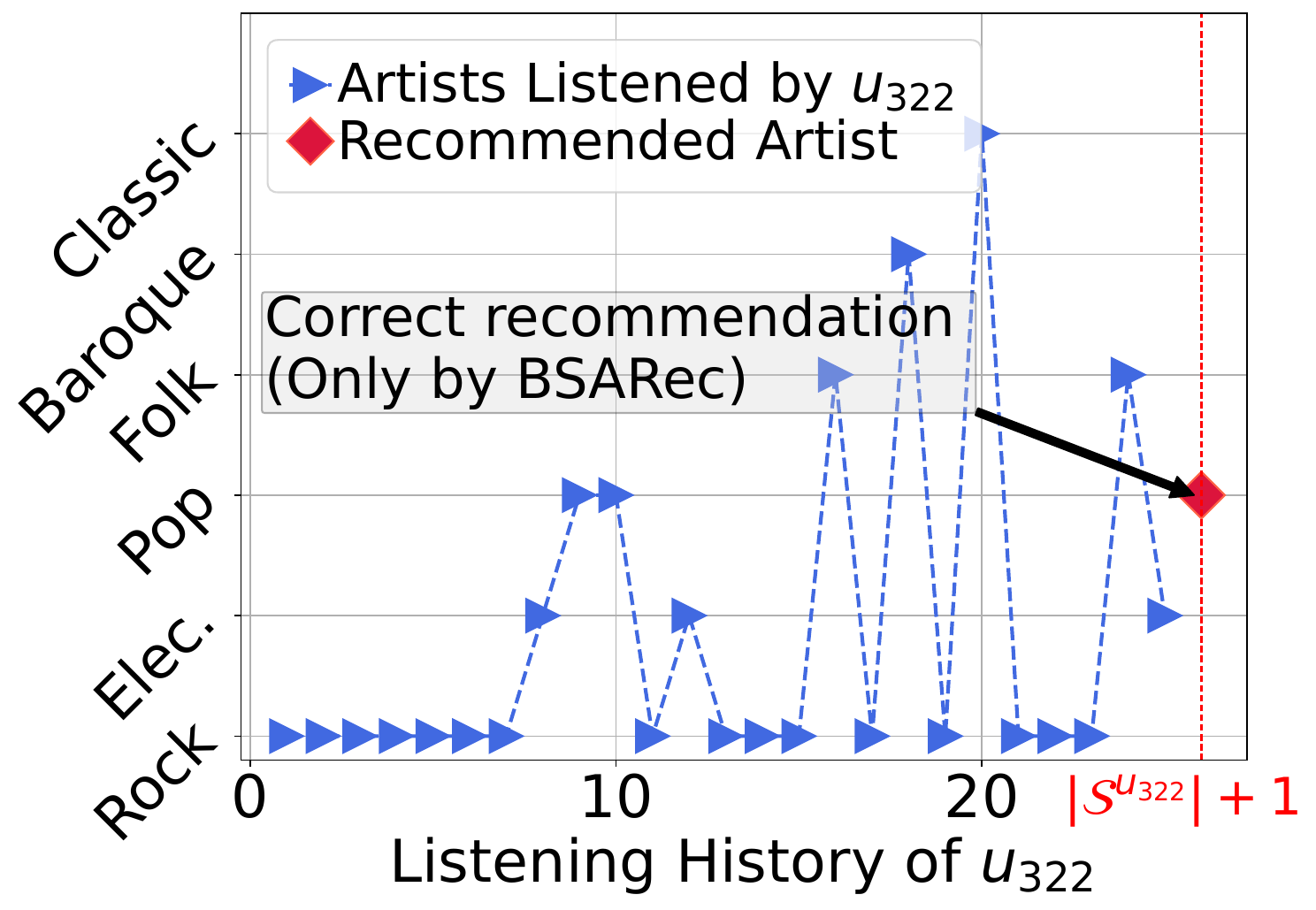}}
    \caption{(a) Visualization of learned $\bm{\beta}$, and (b) an example recommendation in LastFM. The y-axis represents the genre of the artist the user listened to.}
    \label{fig:beta}
\end{figure}

\paragraph{Case Study} 
We introduce a case study obtained from our experiment.
In Fig.~\ref{fig:beta} (b), we analyze one of the heavy users in LastFM. The user $u_{322}$ constantly listens to artists, mainly in the rock genre. In other models, $u_{322}$ cannot capture sudden interaction changes in the next step. Only BSARec recommends an artist from the pop genre as the next artist $u_{322}$ will listen to. This shows that BSARec can capture high-frequency signals that are abrupt changes in user preference.

\subsection{Model Complexity and Runtime Analyses}
To evaluate the overhead of BSARec, we evaluate the number of parameters and runtime per epoch during training. The results are shown in Table~\ref{tab:runtime}. Overall, BSARec increases total parameters marginally. BSARec is actually faster to train than FEARec and DuoRec using contrastive learning. For ML-1M, BSARec is 7.02\% slower than SASRec, but considering the performance difference, it is not a big deal.

\begin{table}[t]
    \centering
    \setlength{\tabcolsep}{3.2pt}
    \begin{tabular}{l cccccc}\toprule
    \multirow{2}{*}{Methods}    && \multicolumn{2}{c}{Beauty}     && \multicolumn{2}{c}{ML-1M}  \\ 
                                \cmidrule(lr){3-4}\cmidrule(lr){6-7}
                                && \# params & s/epoch              && \# params & s/epoch  \\ 
    \midrule
    BSARec                      && 878,208 & 12.75                   && 322,368 & 20.73 \\
    \midrule
    SASRec                      && 877,824 & 10.41                  && 321,984 & 19.37 \\
    DuoRec                      && 877,824 & 19.26                   && 321,984 & 32.33 \\
    FEARec                      && 877,824 & 156.83                   && 321,984 & 278.24 \\
        \bottomrule
    \end{tabular}
    \caption{The number of parameters and training time (runtime per epoch) on Beauty and ML-1M. More results in other datasets are in Appendix.}
    \label{tab:runtime}
\end{table}

\section{Related Work}

\subsection{Sequential Recommendation}
In SR, the primary objective is to recommend the next item based on the sequential patterns inherent in the user's historical interactions. 
FPMC~\cite{rendle2010fpmc} in SR incorporates Markov Chains to capture item-item transitions. Fossil~\cite{he2016fossil} extends the approach to consider higher-order item transitions, improving its predictive capabilities. 

Several notable works have been conducted in this area, each presenting distinct approaches. Early approaches~\cite{rendle2010fpmc,he2016fossil} tried to improve prediction by utilizing Markov Chains to transition between items in the SR. Another avenue of SR leverages convolutional neural networks for sequence modeling, as seen in Caser~\cite{tang2018caser}. Caser treats the embedding matrix of items in the sequence as an image and applies convolution operators to capture local item-item interactions effectively.

The advancements in deep neural network-based SR methods have also made a profound impact on SR, leading to the adoption of RNNs and self-attention mechanisms. For instance, GRU4Rec~\cite{hidasi2016gru4rec} proposes the utilization of GRUs. The success of Transformer-based models, exemplified by Transformer, has further motivated researchers to explore the potential of self-attention in SR. Notably, SASRec~\cite{kang2018sasrec} and BERT4Rec~\cite{sun2019bert4rec} have demonstrated the efficacy of self-attention. These works signify the continued pursuit of enhanced SR methods by integrating self-attention.

With their success, SR models are actively studied~\cite{qiu2022duorec,zhou2022fmlprec,du2023fearec,lin2023mixed,zhou2023attention,yue2023linear,liu2023diffusion,jiang2023adamct}.
Recently, contrastive learning has been used as an aid to improve SR performance. DuoRec~\cite{qiu2022duorec} uses unsupervised model-level augmentation and supervised semantic positive samples for contrastive learning. FMLPRec~\cite{zhou2022fmlprec} proposes a filter-enhanced MLP. 
This approach utilizes a global filter to eliminate frequency domain noise. However, the global filter tends to assign more significance to lower frequencies while undervaluing relatively higher frequencies. 
FEARec~\cite{du2023fearec} is a contrastive learning-based model that uses time domain attention and autocorrelation.
AC-TSR framework~\cite{zhou2023attention} calibrates unreliable attention weights generated by existing Trransformer-based SR models.
AdaMCT~\cite{jiang2023adamct}, which appears at the same time as our work, incorporates locality-induced bias into the Transformer using a local convolutional filter.

\subsection{Oversmoothing and Transformers}
The concept of oversmoothing was first presented by~\citet{li2018deeper} in the field of graph research. Intuitively, the expression converges to a constant after repeatedly exchanging messages with neighbors as the layer of graph neural networks goes to infinity, and research is active to solve this problem~\cite{rusch2022gcon,choi2023gread}. Coincidentally, a parallel occurrence to oversmoothing is observed in Transformers. Early work empirically attributes this to attention collapse or patch or token uniformity~\cite{zhou2021deepvit,gong2021vision}. \citet{dong2021attention} also reveals that the pure Transformer output converges to a rank 1 matrix. There have been several attempts in computer vision to solve this problem~\cite{wang2022antioversmoothing,guo2023contranorm,choi2023graph}, but in SR, there is only one study that solves fast singular value decay~\cite{fan2023addressing}.

\section{Conclusion}
This paper delves into the realm of sequential recommendation (SR) built upon Transformers, an avenue that has garnered substantial success and popularity. The self-attention within Transformers encounters limitations stemming from insufficient inductive bias and its low-pass filtering properties. We also reveal the oversmoothing due to this low-pass filter in SR.
To address this, we introduce BSARec, which uses a combination of attentive inductive bias and vanilla self-attention and integrates low and high-frequencies to mitigate oversmoothing. By understanding and addressing the limitations of self-attention, BSARec significantly advances SR. Our model surpasses 7 baseline methods across 6 datasets in recommendation performance. In future work, we aim to delve deeper into the frequency dynamics of the self-attention for SR.

\section*{Acknowledgements}
Noseong Park is the corresponding author. 
This work was supported by an IITP grant funded by the Korean government (MSIT) (No.2020-0-01361, Artificial Intelligence Graduate School Program (Yonsei University); No.2022-0-00113, Developing a Sustainable Collaborative Multi-modal Lifelong Learning Framework).

\bibliography{ref}

\begin{thebibliography}{52}
\providecommand{\natexlab}[1]{#1}

\bibitem[{Chen et~al.(2019)Chen, Zhao, Li, Huang, and Ou}]{chen2019behavior}
Chen, Q.; Zhao, H.; Li, W.; Huang, P.; and Ou, W. 2019.
\newblock Behavior sequence transformer for e-commerce recommendation in alibaba.
\newblock In \emph{Proceedings of the 1st International Workshop on Deep Learning Practice for High-dimensional Sparse Data}, 1--4.

\bibitem[{Choi et~al.(2023{\natexlab{a}})Choi, Hong, Park, and Cho}]{choi2023bspm}
Choi, J.; Hong, S.; Park, N.; and Cho, S.-B. 2023{\natexlab{a}}.
\newblock Blurring-Sharpening Process Models for Collaborative Filtering.
\newblock In \emph{SIGIR}.

\bibitem[{Choi et~al.(2023{\natexlab{b}})Choi, Hong, Park, and Cho}]{choi2023gread}
Choi, J.; Hong, S.; Park, N.; and Cho, S.-B. 2023{\natexlab{b}}.
\newblock GREAD: Graph Neural Reaction-Diffusion Networks.
\newblock In \emph{ICML}.

\bibitem[{Choi, Jeon, and Park(2021)}]{choi2021ltocf}
Choi, J.; Jeon, J.; and Park, N. 2021.
\newblock LT-OCF: Learnable-Time ODE-based Collaborative Filtering.
\newblock In \emph{CIKM}.

\bibitem[{Choi et~al.(2023{\natexlab{c}})Choi, Wi, Kim, Shin, Lee, Trask, and Park}]{choi2023graph}
Choi, J.; Wi, H.; Kim, J.; Shin, Y.; Lee, K.; Trask, N.; and Park, N. 2023{\natexlab{c}}.
\newblock Graph Convolutions Enrich the Self-Attention in Transformers!
\newblock \emph{arXiv preprint arXiv:2312.04234}.

\bibitem[{Choi et~al.(2023{\natexlab{d}})Choi, Wi, Lee, Cho, Lee, and Park}]{choi2023rdgcl}
Choi, J.; Wi, H.; Lee, C.; Cho, S.-B.; Lee, D.; and Park, N. 2023{\natexlab{d}}.
\newblock {RDGCL}: Reaction-Diffusion Graph Contrastive Learning for Recommendation.
\newblock \emph{arXiv preprint arXiv:2312.16563}.

\bibitem[{Dong, Cordonnier, and Loukas(2021)}]{dong2021attention}
Dong, Y.; Cordonnier, J.-B.; and Loukas, A. 2021.
\newblock Attention is not all you need: Pure attention loses rank doubly exponentially with depth.
\newblock In \emph{ICML}, 2793--2803. PMLR.

\bibitem[{Dosovitskiy et~al.(2020)Dosovitskiy, Beyer, Kolesnikov, Weissenborn, Zhai, Unterthiner, Dehghani, Minderer, Heigold, Gelly et~al.}]{dosovitskiy2020image}
Dosovitskiy, A.; Beyer, L.; Kolesnikov, A.; Weissenborn, D.; Zhai, X.; Unterthiner, T.; Dehghani, M.; Minderer, M.; Heigold, G.; Gelly, S.; et~al. 2020.
\newblock An image is worth 16x16 words: Transformers for image recognition at scale.
\newblock \emph{ICLR}.

\bibitem[{Du et~al.(2023)Du, Yuan, Zhao, Qu, Zhuang, Liu, Liu, and Sheng}]{du2023fearec}
Du, X.; Yuan, H.; Zhao, P.; Qu, J.; Zhuang, F.; Liu, G.; Liu, Y.; and Sheng, V.~S. 2023.
\newblock Frequency Enhanced Hybrid Attention Network for Sequential Recommendation.
\newblock In \emph{SIGIR}, 78--88.

\bibitem[{Fan et~al.(2023)Fan, Liu, Peng, and Yu}]{fan2023addressing}
Fan, Z.; Liu, Z.; Peng, H.; and Yu, P.~S. 2023.
\newblock Addressing the Rank Degeneration in Sequential Recommendation via Singular Spectrum Smoothing.
\newblock \emph{arXiv preprint arXiv:2306.11986}.

\bibitem[{Gao et~al.(2023)Gao, Zheng, Li, Li, Qin, Piao, Quan, Chang, Jin, He et~al.}]{gao2023surveyrec}
Gao, C.; Zheng, Y.; Li, N.; Li, Y.; Qin, Y.; Piao, J.; Quan, Y.; Chang, J.; Jin, D.; He, X.; et~al. 2023.
\newblock A survey of graph neural networks for recommender systems: Challenges, methods, and directions.
\newblock \emph{ACM Transactions on Recommender Systems}, 1(1): 1--51.

\bibitem[{Gong et~al.(2021)Gong, Wang, Li, Chandra, and Liu}]{gong2021vision}
Gong, C.; Wang, D.; Li, M.; Chandra, V.; and Liu, Q. 2021.
\newblock Vision transformers with patch diversification.
\newblock \emph{arXiv preprint arXiv:2104.12753}.

\bibitem[{Guo et~al.(2023)Guo, Wang, Du, and Wang}]{guo2023contranorm}
Guo, X.; Wang, Y.; Du, T.; and Wang, Y. 2023.
\newblock Contranorm: A contrastive learning perspective on oversmoothing and beyond.
\newblock In \emph{ICLR}.

\bibitem[{Hansen et~al.(2020)Hansen, Hansen, Maystre, Mehrotra, Brost, Tomasi, and Lalmas}]{hansen2020contextual}
Hansen, C.; Hansen, C.; Maystre, L.; Mehrotra, R.; Brost, B.; Tomasi, F.; and Lalmas, M. 2020.
\newblock Contextual and sequential user embeddings for large-scale music recommendation.
\newblock In \emph{RecSys}, 53--62.

\bibitem[{Harper and Konstan(2015)}]{harper2015movielens}
Harper, F.~M.; and Konstan, J.~A. 2015.
\newblock The movielens datasets: History and context.
\newblock \emph{Acm transactions on interactive intelligent systems (tiis)}, 5(4): 1--19.

\bibitem[{He and McAuley(2016)}]{he2016fossil}
He, R.; and McAuley, J. 2016.
\newblock Fusing similarity models with markov chains for sparse sequential recommendation.
\newblock In \emph{ICDM}, 191--200. IEEE.

\bibitem[{He et~al.(2020)He, Deng, Wang, Li, Zhang, and Wang}]{He20LightGCN}
He, X.; Deng, K.; Wang, X.; Li, Y.; Zhang, Y.; and Wang, M. 2020.
\newblock LightGCN: Simplifying and Powering Graph Convolution Network for Recommendation.
\newblock In \emph{SIGIR}.

\bibitem[{He and Wai(2021)}]{he2021identifying}
He, Y.; and Wai, H.-T. 2021.
\newblock Identifying first-order lowpass graph signals using perron frobenius theorem.
\newblock In \emph{ICASSP 2021-2021 IEEE International Conference on Acoustics, Speech and Signal Processing (ICASSP)}, 5285--5289. IEEE.

\bibitem[{Hidasi et~al.(2016)Hidasi, Karatzoglou, Baltrunas, and Tikk}]{hidasi2016gru4rec}
Hidasi, B.; Karatzoglou, A.; Baltrunas, L.; and Tikk, D. 2016.
\newblock Session-based recommendations with recurrent neural networks.
\newblock In \emph{ICLR}.

\bibitem[{Hong et~al.(2022)Hong, Jo, Kook, Jung, Wi, Park, and Cho}]{hong2022timekit}
Hong, S.; Jo, M.; Kook, S.; Jung, J.; Wi, H.; Park, N.; and Cho, S.-B. 2022.
\newblock {TimeKit}: A Time-series Forecasting-based Upgrade Kit for Collaborative Filtering.
\newblock In \emph{2022 IEEE International Conference on Big Data (Big Data)}, 565--574. IEEE.

\bibitem[{Huang et~al.(2018)Huang, Qian, Fang, Sang, and Xu}]{huang2018csan}
Huang, X.; Qian, S.; Fang, Q.; Sang, J.; and Xu, C. 2018.
\newblock {CSAN}: Contextual self-attention network for user sequential recommendation.
\newblock In \emph{ACM MM}, 447--455.

\bibitem[{Jiang et~al.(2023)Jiang, Zhang, Luo, Li, Kim, Zhang, Wang, Xie, and Kim}]{jiang2023adamct}
Jiang, J.; Zhang, P.; Luo, Y.; Li, C.; Kim, J.~B.; Zhang, K.; Wang, S.; Xie, X.; and Kim, S. 2023.
\newblock AdaMCT: adaptive mixture of CNN-transformer for sequential recommendation.
\newblock In \emph{CIKM}.

\bibitem[{Jiang et~al.(2016)Jiang, Qian, Mei, and Fu}]{jiang2016personalized}
Jiang, S.; Qian, X.; Mei, T.; and Fu, Y. 2016.
\newblock Personalized travel sequence recommendation on multi-source big social media.
\newblock \emph{IEEE Transactions on Big Data}, 2(1): 43--56.

\bibitem[{Kang and McAuley(2018)}]{kang2018sasrec}
Kang, W.-C.; and McAuley, J. 2018.
\newblock Self-attentive sequential recommendation.
\newblock In \emph{ICDM}, 197--206. IEEE.

\bibitem[{Kong et~al.(2022)Kong, Kim, Jeon, Choi, Lee, Park, and Kim}]{kong2022hmlet}
Kong, T.; Kim, T.; Jeon, J.; Choi, J.; Lee, Y.-C.; Park, N.; and Kim, S.-W. 2022.
\newblock Linear, or Non-Linear, That is the Question!
\newblock In \emph{WSDM}, 517--525.

\bibitem[{Krichene and Rendle(2020)}]{krichene2020sampled}
Krichene, W.; and Rendle, S. 2020.
\newblock On sampled metrics for item recommendation.
\newblock In \emph{Proceedings of the 26th ACM SIGKDD international conference on knowledge discovery \& data mining}, 1748--1757.

\bibitem[{Lee, Kim, and Lee(2018)}]{lee2018goccf}
Lee, Y.-C.; Kim, S.-W.; and Lee, D. 2018.
\newblock {gOCCF}: Graph-theoretic one-class collaborative filtering based on uninteresting items.
\newblock In \emph{AAAI}, volume~32.

\bibitem[{Li, Wang, and McAuley(2020)}]{li2020tisrec}
Li, J.; Wang, Y.; and McAuley, J. 2020.
\newblock Time interval aware self-attention for sequential recommendation.
\newblock In \emph{WSDM}, 322--330.

\bibitem[{Li, Han, and Wu(2018)}]{li2018deeper}
Li, Q.; Han, Z.; and Wu, X.-M. 2018.
\newblock Deeper insights into graph convolutional networks for semi-supervised learning.
\newblock In \emph{AAAI}.

\bibitem[{Lin et~al.(2023)Lin, Gao, Zheng, Chang, Niu, Song, Gai, Li, Jin, Li et~al.}]{lin2023mixed}
Lin, G.; Gao, C.; Zheng, Y.; Chang, J.; Niu, Y.; Song, Y.; Gai, K.; Li, Z.; Jin, D.; Li, Y.; et~al. 2023.
\newblock Mixed Attention Network for Cross-domain Sequential Recommendation.
\newblock \emph{arXiv preprint arXiv:2311.08272}.

\bibitem[{Liu et~al.(2023)Liu, Yan, Zhao, Du, Guo, Tang, and Tian}]{liu2023diffusion}
Liu, Q.; Yan, F.; Zhao, X.; Du, Z.; Guo, H.; Tang, R.; and Tian, F. 2023.
\newblock Diffusion Augmentation for Sequential Recommendation.
\newblock In \emph{CIKM}, 1576--1586.

\bibitem[{McAuley et~al.(2015)McAuley, Targett, Shi, and Van Den~Hengel}]{mcauley2015image}
McAuley, J.; Targett, C.; Shi, Q.; and Van Den~Hengel, A. 2015.
\newblock Image-based recommendations on styles and substitutes.
\newblock In \emph{Proceedings of the 38th international ACM SIGIR conference on research and development in information retrieval}, 43--52.

\bibitem[{Meyer and Stewart(2023)}]{meyer2023matrix}
Meyer, C.~D.; and Stewart, I. 2023.
\newblock \emph{Matrix analysis and applied linear algebra}.
\newblock SIAM.

\bibitem[{Qiu et~al.(2022)Qiu, Huang, Yin, and Wang}]{qiu2022duorec}
Qiu, R.; Huang, Z.; Yin, H.; and Wang, Z. 2022.
\newblock Contrastive learning for representation degeneration problem in sequential recommendation.
\newblock In \emph{WSDM}, 813--823.

\bibitem[{Rendle, Freudenthaler, and Schmidt-Thieme(2010)}]{rendle2010fpmc}
Rendle, S.; Freudenthaler, C.; and Schmidt-Thieme, L. 2010.
\newblock Factorizing personalized markov chains for next-basket recommendation.
\newblock In \emph{TheWebConf (former WWW)}, 811--820.

\bibitem[{Rusch et~al.(2022)Rusch, Chamberlain, Rowbottom, Mishra, and Bronstein}]{rusch2022gcon}
Rusch, T.~K.; Chamberlain, B.; Rowbottom, J.; Mishra, S.; and Bronstein, M. 2022.
\newblock Graph-Coupled Oscillator Networks.
\newblock In \emph{ICML}, volume 162, 18888--18909.

\bibitem[{Sandryhaila and Moura(2014)}]{sandryhaila2014discrete}
Sandryhaila, A.; and Moura, J.~M. 2014.
\newblock Discrete signal processing on graphs: Frequency analysis.
\newblock \emph{IEEE Transactions on Signal Processing}, 62(12): 3042--3054.

\bibitem[{Schedl et~al.(2018)Schedl, Zamani, Chen, Deldjoo, and Elahi}]{schedl2018current}
Schedl, M.; Zamani, H.; Chen, C.-W.; Deldjoo, Y.; and Elahi, M. 2018.
\newblock Current challenges and visions in music recommender systems research.
\newblock \emph{International Journal of Multimedia Information Retrieval}, 7: 95--116.

\bibitem[{Sun et~al.(2019)Sun, Liu, Wu, Pei, Lin, Ou, and Jiang}]{sun2019bert4rec}
Sun, F.; Liu, J.; Wu, J.; Pei, C.; Lin, X.; Ou, W.; and Jiang, P. 2019.
\newblock {BERT4Rec}: Sequential recommendation with bidirectional encoder representations from transformer.
\newblock In \emph{CIKM}, 1441--1450.

\bibitem[{Tang and Wang(2018)}]{tang2018caser}
Tang, J.; and Wang, K. 2018.
\newblock Personalized top-n sequential recommendation via convolutional sequence embedding.
\newblock In \emph{WSDM}, 565--573.

\bibitem[{Vaswani et~al.(2017)Vaswani, Shazeer, Parmar, Uszkoreit, Jones, Gomez, Kaiser, and Polosukhin}]{vaswani2017attention}
Vaswani, A.; Shazeer, N.; Parmar, N.; Uszkoreit, J.; Jones, L.; Gomez, A.~N.; Kaiser, {\L}.; and Polosukhin, I. 2017.
\newblock Attention is all you need.
\newblock In \emph{NeurIPS}.

\bibitem[{Wang et~al.(2022)Wang, Zheng, Chen, and Wang}]{wang2022antioversmoothing}
Wang, P.; Zheng, W.; Chen, T.; and Wang, Z. 2022.
\newblock Anti-Oversmoothing in Deep Vision Transformers via the Fourier Domain Analysis: From Theory to Practice.
\newblock In \emph{ICLR}.

\bibitem[{Wu, Cai, and Wang(2020)}]{wu2020deja}
Wu, J.; Cai, R.; and Wang, H. 2020.
\newblock D{\'e}j{\`a} vu: A contextualized temporal attention mechanism for sequential recommendation.
\newblock In \emph{TheWebConf (former WWW)}, 2199--2209.

\bibitem[{Wu et~al.(2020)Wu, Li, Hsieh, and Sharpnack}]{wu2020sse}
Wu, L.; Li, S.; Hsieh, C.-J.; and Sharpnack, J. 2020.
\newblock SSE-PT: Sequential recommendation via personalized transformer.
\newblock In \emph{RecSys}, 328--337.

\bibitem[{Wu et~al.(2022)Wu, Sun, Zhang, Xie, and Cui}]{wu2022surveyrec}
Wu, S.; Sun, F.; Zhang, W.; Xie, X.; and Cui, B. 2022.
\newblock Graph neural networks in recommender systems: a survey.
\newblock \emph{ACM Computing Surveys}, 55(5): 1--37.

\bibitem[{Ying et~al.(2018)Ying, He, Chen, Eksombatchai, Hamilton, and Leskovec}]{Rex2018pinsage}
Ying, R.; He, R.; Chen, K.; Eksombatchai, P.; Hamilton, W.~L.; and Leskovec, J. 2018.
\newblock Graph Convolutional Neural Networks for Web-Scale Recommender Systems.
\newblock In \emph{KDD}.

\bibitem[{Yue et~al.(2023)Yue, Wang, He, Zeng, McAuley, and Wang}]{yue2023linear}
Yue, Z.; Wang, Y.; He, Z.; Zeng, H.; McAuley, J.; and Wang, D. 2023.
\newblock Linear Recurrent Units for Sequential Recommendation.
\newblock \emph{arXiv preprint arXiv:2310.02367}.

\bibitem[{Zhang et~al.(2019)Zhang, Zhao, Liu, Sheng, Xu, Wang, Liu, Zhou et~al.}]{zhang2019feature}
Zhang, T.; Zhao, P.; Liu, Y.; Sheng, V.~S.; Xu, J.; Wang, D.; Liu, G.; Zhou, X.; et~al. 2019.
\newblock Feature-level Deeper Self-Attention Network for Sequential Recommendation.
\newblock In \emph{IJCAI}, 4320--4326.

\bibitem[{Zhou et~al.(2021)Zhou, Kang, Jin, Yang, Lian, Jiang, Hou, and Feng}]{zhou2021deepvit}
Zhou, D.; Kang, B.; Jin, X.; Yang, L.; Lian, X.; Jiang, Z.; Hou, Q.; and Feng, J. 2021.
\newblock Deepvit: Towards deeper vision transformer.
\newblock \emph{arXiv preprint arXiv:2103.11886}.

\bibitem[{Zhou et~al.(2020)Zhou, Wang, Zhao, Zhu, Wang, Zhang, Wang, and Wen}]{zhou2020s3}
Zhou, K.; Wang, H.; Zhao, W.~X.; Zhu, Y.; Wang, S.; Zhang, F.; Wang, Z.; and Wen, J.-R. 2020.
\newblock S3-rec: Self-supervised learning for sequential recommendation with mutual information maximization.
\newblock In \emph{CIKM}, 1893--1902.

\bibitem[{Zhou et~al.(2022)Zhou, Yu, Zhao, and Wen}]{zhou2022fmlprec}
Zhou, K.; Yu, H.; Zhao, W.~X.; and Wen, J.-R. 2022.
\newblock Filter-enhanced MLP is all you need for sequential recommendation.
\newblock In \emph{TheWebConf (former WWW)}, 2388--2399.

\bibitem[{Zhou et~al.(2023)Zhou, Ye, Xie, Gao, Wang, Kim, You, and Kim}]{zhou2023attention}
Zhou, P.; Ye, Q.; Xie, Y.; Gao, J.; Wang, S.; Kim, J.~B.; You, C.; and Kim, S. 2023.
\newblock Attention Calibration for Transformer-based Sequential Recommendation.
\newblock In \emph{CIKM}, 3595--3605.

\end{thebibliography}

\clearpage
\appendix

\section{More Preliminaries on Fourier Transform}
We provide the detailed background of Fourier transform. Here, we only consider discrete Fourier transform (DFT) on real-value domain $\mathcal{F}: \mathbb{R}^N\rightarrow\mathbb{C}^{N}$. DFT is a transformation used to convert a discrete time-domain signal into a discrete frequency-domain representation. This transformation is particularly useful in signal processing and many other fields. In this paper, we use the following expression of the Fourier basis, which is the row of this matrix: $\bm{f}_j = [e^{2\pi i(j-1)\cdot0} \ldots e^{2\pi i(j-1)(N-1)}]^{\mathtt{T}}/\sqrt{N}\in\mathbb{R}^N$, where $i$ is the imaginary unit and $j$ denotes the $j$-th row. The DFT is often used in its matrix form as a transformation matrix, which can be applied to a signal through matrix multiplication. We can define the DFT matrix as follows:
\begin{align}\mathbf{F}=\frac{1}{\sqrt{N}}
    \begin{bmatrix}
        1 & 1 & \ldots & 1\\
        1 & e^{2\pi i} & \ldots & e^{2\pi i (N-1)}\\
        \vdots & \vdots & \ddots & \vdots\\
        1 & e^{2\pi i(j-1)\cdot1} & \ldots & e^{2\pi i(j-1)\cdot(N-1)}\\
        \vdots & \vdots & \ddots & \vdots\\
        1 & e^{2\pi i(N-1)} & \ldots & e^{2\pi i(N-1)^2}
    \end{bmatrix},
\end{align}
and its inverse Fourier transform is DFT. Specifically, when the matrix representation of the IDFT is the conjugate transpose (Hermitian) of the DFT matrix, scaled by a factor of $1/\sqrt{N}$, the transformation becomes unitary. The IDFT matrix can be expressed as:
\begin{align}
    \mathbf{F}^{-1} = \frac{1}{\sqrt{N}}\mathbf{F}^H,
\end{align}
where $\mathbf{F}^H$ denotes the conjugate transpose of $\mathbf{F}$. The factor $1/\sqrt{N}$ ensures that the transformation is norm-preserving, meaning that the energy of the signal remains invariant under both the DFT and IDFT. This normalization underpins the unitary nature of the transformation and emphasizes the duality between the time and frequency domains in discrete signal processing. 

This relationship emphasizes that the IDFT is not an isolated operation, but rather intimately linked with the DFT through the conjugate transpose operation and a normalization factor. By utilizing the DFT matrix $\mathbf{F}$ in this manner, we are able to conduct the IDFT without resorting to a separate transformation matrix.

\section{Definition of a Low-Pass Filter}
In this paper, a specific type of filter that preserves only the low-frequency components while reducing the remaining high-frequency components is called a low-pass filter. More precisely, we define a low-pass filter.
\begin{definition}\label{def:lpf}
Given a transformation $f: \mathbb{R}^N \rightarrow \mathbb{R}^N$, with indicating the application of $f$ for $t$ times, the $f$ is a low-pass filter if and only if, for all $\bm{x}\in\mathbb{R}^N$:
\begin{align}
    \lim_{t\rightarrow \infty} \frac{||HFC(f^t(\bm{x}))||_2}{||LFC(f^t(\bm{x}))||_2}=0.
\end{align}

\end{definition}
Definition~\ref{def:lpf} provides a criterion to evaluate the frequency behavior of a transformation. In essence, if after recurrently applying the transformation $f$ to a $\bm{x}$, and then taking its DFT, the high-frequency components are diminished compared to the low-frequency components as time advances, then $f$ can be described as exhibiting a low-pass behavior.

\section{Proof of Theorem~\ref{theorem:lpf}}\label{app:proof}

\begin{customthm}
Let $\mathbf{A} =  \textrm{softmax}(\mathbf{Q}\mathbf{K}^{\mathtt{T}}/\sqrt{d} )$. Then $\mathbf{A}$ inherently acts as a low-pass filter. For all $\bm{x}\in\mathbb{R}^N$, in other words,
\begin{align}
    \lim_{t\rightarrow \infty} \frac{||HFC(f^t(\bm{x}))||_2}{||LFC(f^t(\bm{x}))||_2}=0.
\end{align}
\end{customthm}

\begin{proof}
Given the matrix $\mathbf{A}$  defined by the softmax function, it has non-negative entries, and the sum of each row is unity. We aim to describe the evolution of the $\bm{x}$ under repeated application of $\mathbf{A}$.

Let's denote the Jordan Canonical Form of $\mathbf{A}$ as $\mathbf{J}$, with the similarity transformation represented by the matrix $\mathbf{P}$, such that:
\begin{align}
    \mathbf{A}=\mathbf{P}\mathbf{J}\mathbf{P}^{-1},
\end{align}where $\mathbf{J}$ is block-diagonal, with each block corresponding to an eigenvalue and its associated Jordan chains. The largest eigenvalue, by the Perron-Frobenius theorem, is real, non-negative, and dominant. Let's denote this eigenvalue by $\lambda_1$.

Now, consider the repeated application of $\mathbf{A}$:
\begin{align}
    f^t(\bm{x}) = \mathbf{A}^t\bm{x}= (\mathbf{P}\mathbf{J}\mathbf{P}^{-1})^t\bm{x}.
\end{align}
Expanding using the binomial theorem and taking into account the structure of the Jordan blocks, we can see that the dominant behavior for large $t$ will be $\lambda^t_1$. Other terms involving smaller eigenvalues or higher powers of $t$ in the Jordan blocks will become negligible over time, compared to the term with $\lambda^t_1$.

Expressing the transformation in the frequency domain, high-frequency components attenuate faster than the primary low-frequency component. This is due to the term $\lambda^t_1$ becoming overwhelmingly dominant as $t$ grows, causing other components to diminish in comparison.

Thus, in the context of our filter definitions, it becomes evident that:
\begin{align}
\lim_{t\rightarrow \infty} \frac{||HFC[f^t(\bm{x})-\lambda^t_1\mathbf{v}_1]||_2} {||LFC[\lambda^t_1\mathbf{v}_1]||_2}=0
\end{align}
Here, $\mathbf{v}_1$ is the generalized eigenvector corresponding to $\lambda_1$.

This behavior is emblematic of a low-pass filter, reaffirming the low-pass nature of $\mathbf{A}$. Importantly, this is irrespective of the specific configurations of the input matrices $\mathbf{Q}$ and $\mathbf{K}$.
\end{proof}

\section{Additional Details for Experiments}
\subsection{Details of Datasets}\label{app:data}
We provide 6 benchmark datasets used for our experiments. Table~\ref{tab:data} summarizes the statistical information of the processed datasets. 

\begin{table}[h]
    \centering
    \setlength{\tabcolsep}{2pt}
    \resizebox{0.95\columnwidth}{!}{%
    \begin{tabular}{l ccccc}\toprule
        & \# Users & \# Items & \# Interactions & Avg. Length & Sparsity \\ \midrule
        Beauty      & 22,363   & 12,101   & 198,502       & 8.9        & 99.93\%  \\
        Sports      & 25,598   & 18,357   & 296,337       & 8.3        & 99.95\%  \\
        Toys        & 19,412   & 11,924   & 167,597       & 8.6        & 99.93\%  \\
        Yelp        & 30,431   & 20,033   & 316,354       & 10.4       & 99.95\%  \\
        LastFM      & 1,090    & 3,646    & 52,551        & 48.2       & 98.68\%  \\
        ML-1M       & 6,041    & 3,417    & 999,611       & 165.5      & 95.16\%  \\
        \bottomrule
    \end{tabular}
    }
    \caption{Statistics of the processed datasets}
    \label{tab:data}
\end{table}

\begin{itemize}
    \item Amazon Beauty, Sports, Toys are three sub-categories of Amazon dataset~\cite{mcauley2015image} which contains series of product reviews crawled from the Amazon.com. These datasets are known for its high sparsity and short sequence lengths, and widely used for SR. 
    \item Yelp\footnote{https://www.yelp.com/dataset} is a popular business recommendation dataset. We only treat the transaction records after January 1st, 2019 since it is very large.
    \item ML-1M~\cite{harper2015movielens} is the popular movie recommendation dataset provided by MovieLens\footnote{https://grouplens.org/datasets/movielens/}. It has the longest average interaction length among our datasets.
    \item LastFM\footnote{https://grouplens.org/datasets/hetrec-2011/} contains user interaction with music, such as artist listening records. It is used to recommend musicians to users in SR with long sequence lengths.
\end{itemize}

\subsection{Details of Baselines}\label{app:baselines}
We compare our method with well-known SR baselines with 3 categories. The detailed information of baselines are follows:
\begin{itemize}
    \item RNN or CNN-based sequential models: GRU4Rec~\cite{hidasi2016gru4rec} is a model that incorporate GRU for SR. Caser~\cite{tang2018caser} is a CNN-based method capturing high-order patterns by applying horizontal and vertical convolutional operations for SR.
    \item Transformer-based sequential models: SASRec~\cite{kang2018sasrec} is the first sequential recommender based on the self-attention. It is a popular baseline in SR. BERT4Rec~\cite{sun2019bert4rec} uses a masked item training scheme similar to the masked language model sequential in NLP. The backbone is the bi-directional self-attention mechanism. FMLPRec~\cite{zhou2022fmlprec} is an all-MLP model with learnable filters for SR.
    \item Transformer-based sequential models with contrsastive learning: DuoRec~\cite{qiu2022duorec} uses unsupervised model-level augmentation and supervised semantic positive samples for contrastive learning. FEARec~\cite{du2023fearec} is a SR model based on contrastive learning using time domain attention and autocorrelation.
\end{itemize}

\subsection{Experimental Settings \& Hyperparameters}\label{app:best}
The following software and hardware environments were used for all experiments: \textsc{Ubuntu} 18.04 LTS, \textsc{Python} 3.9.7,  \textsc{PyTorch} 1.8.1, \textsc{Numpy} 1.24.3, \textsc{Scipy} 1.11.1, \textsc{CUDA} 11.1, and \textsc{NVIDIA} Driver 465.19, and i9 CPU, and \textsc{NVIDIA RTX 3090}. 

For reproducibility, we introduce the best hyperparameter configurations for each dataset in Table~\ref{tab:best}. We conducted experiments under the following hyperparameters: the $\alpha$ is in $\{0.1, 0.3, 0.5, 0.7, 0.9\}$, $c$ is in $\{1, 3, 5, 7, 9\}$, and the number of heads in Transformer $h$ is chosen from $\{1, 2, 4\}$. For training, the Adam optimizer is optimized with learning rate in \{\num{5e-4}, \num{1e-3}\}.

\begin{table}[h]
    \centering
    \setlength{\tabcolsep}{2pt}
    \resizebox{0.95\columnwidth}{!}{%
    \begin{tabular}{l cccccc}\toprule
                & Beauty & Sports & Toys & Yelp & LastFM & ML-1M \\ \midrule
    $\alpha$    & 0.7 & 0.3 & 0.7 & 0.7 & 0.9 & 0.3 \\
    $c$         & 5 & 5 & 3 & 3 & 3 & 9 \\
    $h$         & 1 & 4 & 1 & 4 & 1 & 4 \\
    lr          & \num{5e-4} & \num{1e-3} & \num{1e-3} & \num{1e-3} & \num{1e-3} & \num{5e-4} \\
        \bottomrule
    \end{tabular}
    }
    \caption{Best hyperparameters of BSARec on all datasets}
    \label{tab:best}
\end{table}

\section{Ablation Studies}
Table~\ref{tab:abl_full} shows the results of the ablation study on all datasets. For all datasets, BSARec, which utilizes $\mathbf{A}$ and $\mathbf{A}_{\textrm{IB}}$, outperforms ``Only $\mathbf{A}$'' and ``Only $\mathbf{A}_{\textrm{IB}}$''. This results show that both are required to achieve the best accuracy. We also conduct an ablation study on $\bm{\beta}$. We check the effect of the learnable vector $\bm{\beta}$ and the learnable scalar parameter $\beta$. Rescaling high frequencies by the learnable vector $\beta$ has better recommendation performance than using the scalar $\beta$.

\begin{table*}[h]
    \centering
    \setlength{\tabcolsep}{1.5pt}
    \resizebox{1.02\textwidth}{!}{%
    \begin{tabular}{l cccccccccccccccccc}
    \toprule
    \multirow{2}{*}{Methods}    && \multicolumn{2}{c}{Beauty} && \multicolumn{2}{c}{Sports} && \multicolumn{2}{c}{Toys} && \multicolumn{2}{c}{Yelp} && \multicolumn{2}{c}{LastFM} && \multicolumn{2}{c}{ML-1M}  \\ 
    \cmidrule(lr){3-4} \cmidrule(lr){6-7} \cmidrule(lr){9-10} \cmidrule(lr){12-13} \cmidrule(lr){15-16} \cmidrule(lr){18-19}
                                && HR@20 & NDCG@20  &&  HR@20 & NDCG@20     && HR@20 & NDCG@20      &&  HR@20 & NDCG@20     && HR@20 & NDCG@20      &&  HR@20 & NDCG@20 \\ 
    \midrule
    BSARec                      && \textbf{0.1373} & \textbf{0.0703}  && \textbf{0.0858} & \textbf{0.0422}      && \textbf{0.1435} & \textbf{0.0768}      && \textbf{0.0746} & \textbf{0.0302}      && \textbf{0.1174} & \textbf{0.0526}      && \textbf{0.3884} & \textbf{0.1851} \\
    \midrule
    Only $\mathbf{A}$           && 0.1265 & 0.0657  && 0.0779 & 0.0382      && 0.1320 & 0.0720      && 0.0618 & 0.0248      && 0.0899 & 0.0430      && 0.3826 & 0.1846\\
    Only $\mathbf{A}_{\textrm{IB}}$      && 0.1338 & 0.0677  && 0.0857 & 0.0416      && 0.1402 & 0.0744      && 0.0705 & 0.0287      && 0.1009 & 0.0455      && 0.3780 & 0.1807\\
    \midrule
    Scalar $\beta$              && 0.1333 & 0.0685  && 0.0838 & 0.0405      && \textbf{0.1435} & 0.0756      && 0.0707 & 0.0291      && 0.1092 & 0.0497      && 0.3762 & 0.1794\\
    \bottomrule
    \end{tabular}
    }
    \caption{Ablation on all datasets}
    \label{tab:abl_full}
\end{table*}

\section{Sensitivity Analyses}\label{app:sens}
Fig.~\ref{fig:sen_alpha_appendix} shows HR@20 and NDCG@20 by varying $\alpha$ in all datasets. For Beauty, Toys, Yelp and LastFM, a larger value of $\alpha$ is preferred. For Sports and ML-1M, the best accuracy is achieved when $\alpha$. The balance between the self-attention matrix and the inductive bias varies across datasets to these results.

Fig.~\ref{fig:sen_c_appendix} shows HR@20 and NDCG@20 by varying the $c$ in all datasets. For Beauty and Sports, the best accuracy is achieved when $c$ is 5. For Toys, Yelp and LastFM, with $c$ = 3, we can achieve the best accuracy. For ML-1M, a larger value of $c$ is preferred. The number of lowest frequency components varies across datasets according to these results.

\begin{figure}[t]
    \begin{subfigure}[Beauty]{\includegraphics[width=.49\columnwidth]{images/sensitivity_alpha_Beauty.pdf}}
    \end{subfigure}
    \hspace{-1em}
    \begin{subfigure}[Sports]{\includegraphics[width=.49\columnwidth]{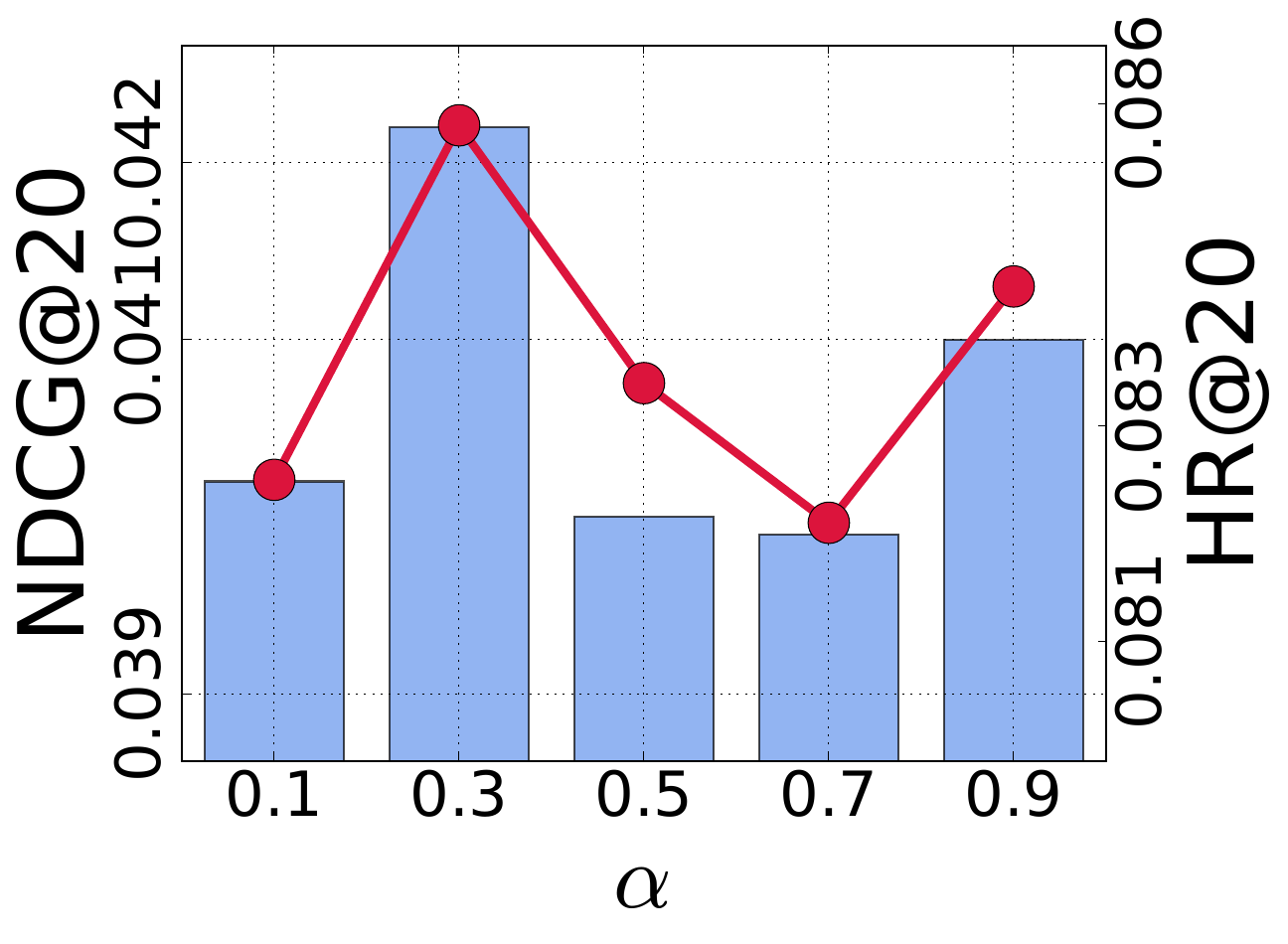}}
    \end{subfigure}
    \begin{subfigure}[Toys]{\includegraphics[width=.49\columnwidth]{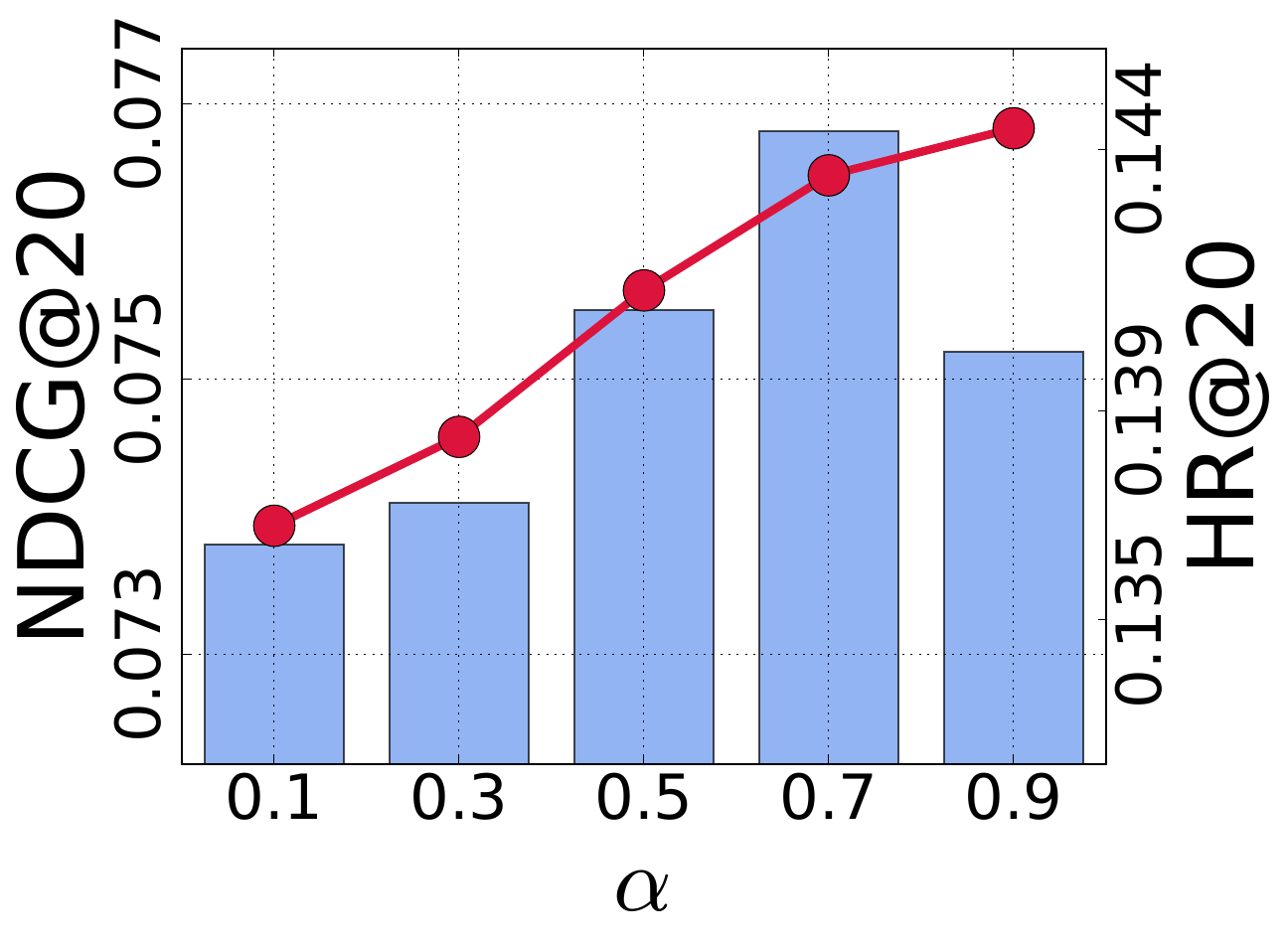}}
    \end{subfigure}
    \hspace{-1em}
    \begin{subfigure}[Yelp]{\includegraphics[width=.49\columnwidth]{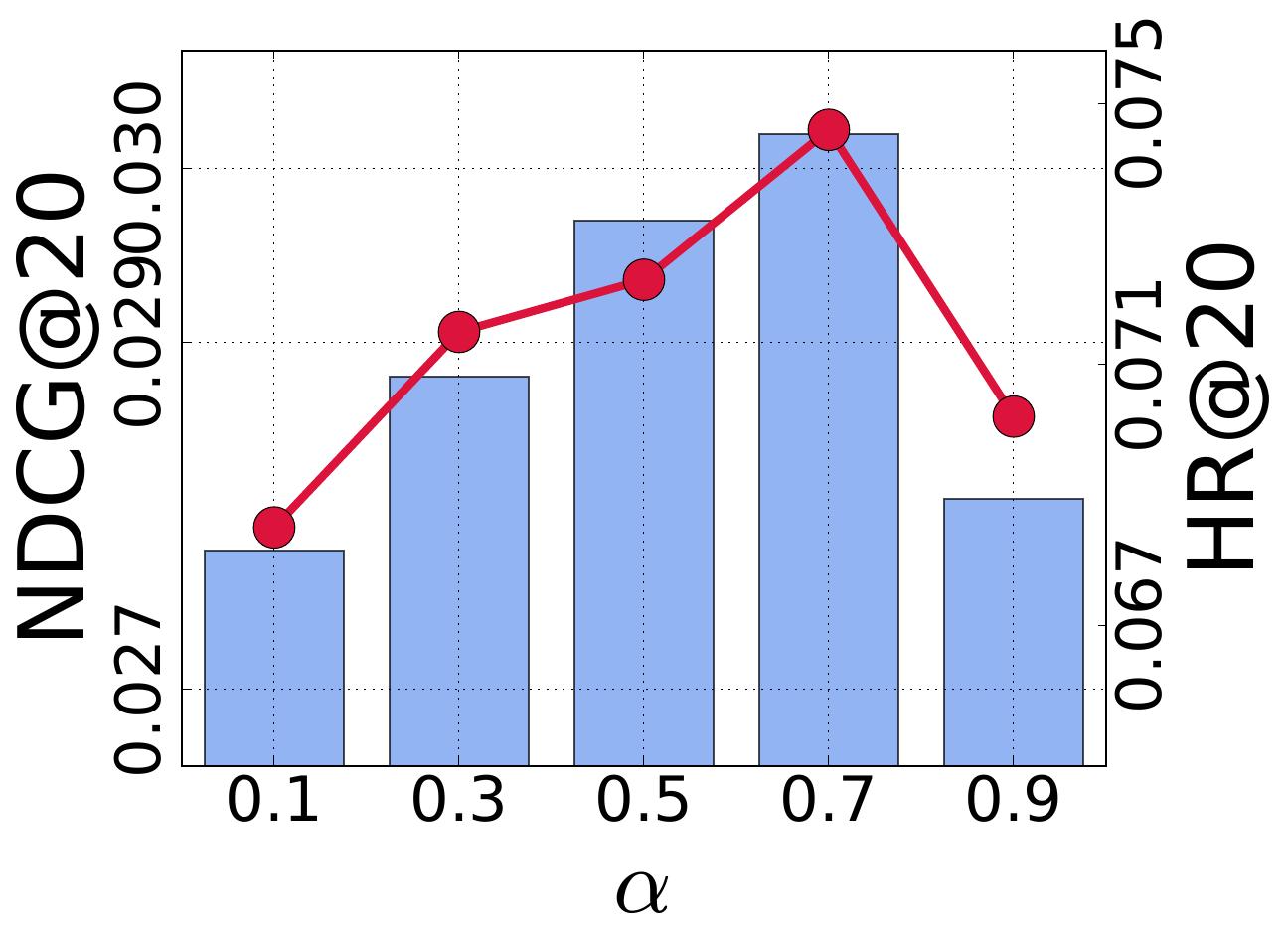}}
    \end{subfigure}
    \begin{subfigure}[LastFM]{\includegraphics[width=.49\columnwidth]{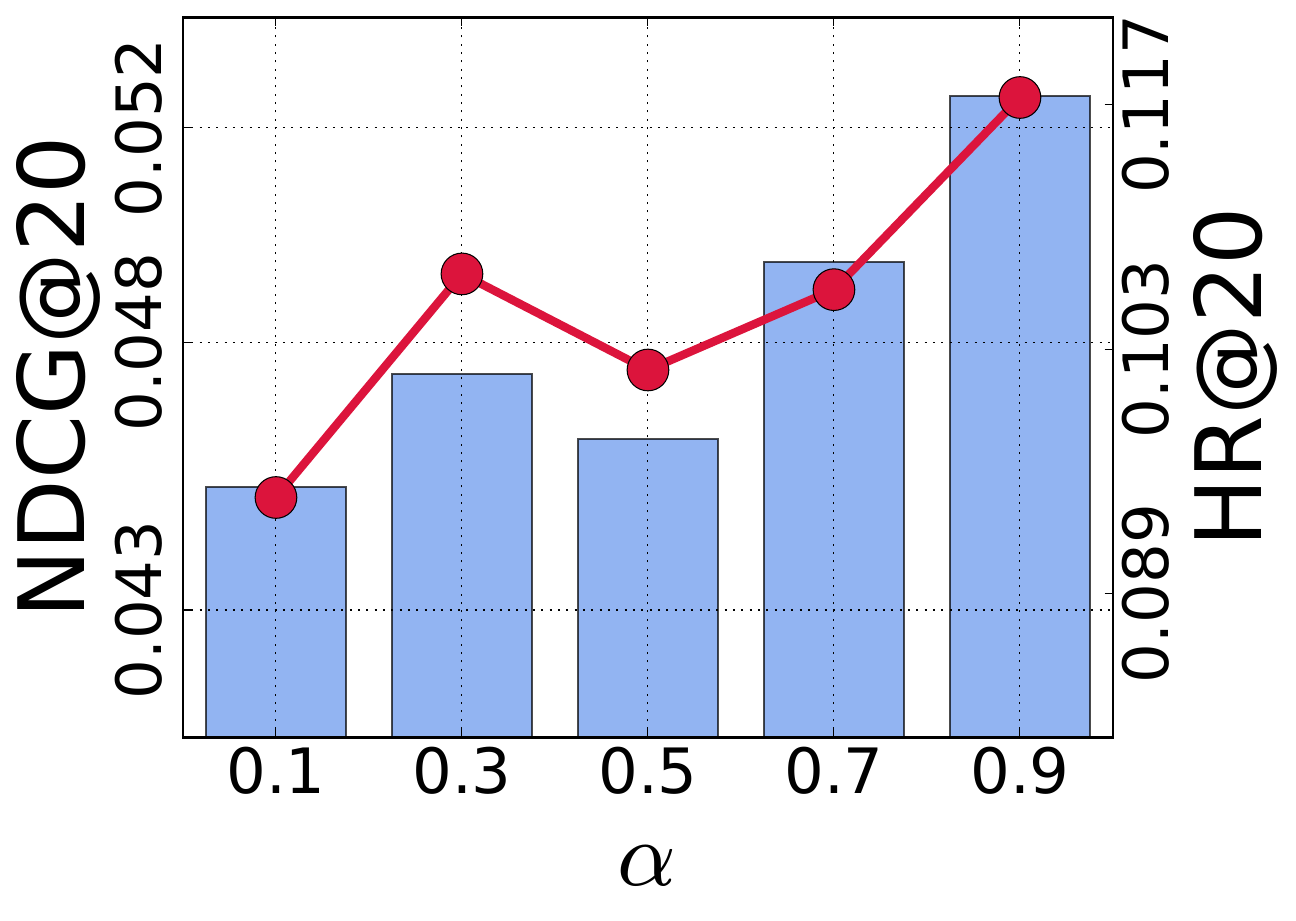}}
    \end{subfigure}
    \hspace{-1em}
    \begin{subfigure}[ML-1M]{\includegraphics[width=.49\columnwidth]{images/sensitivity_alpha_ML-1M.pdf}}
    \end{subfigure}
    \caption{Sensitivity to $\alpha$ on all datasets}
    \label{fig:sen_alpha_appendix}
\end{figure}

\begin{figure}[t]
    \begin{subfigure}[Beauty]{\includegraphics[width=.49\columnwidth]{images/sensitivity_c_Beauty.pdf}}
    \end{subfigure}
    \hspace{-1em}
    \begin{subfigure}[Sports]{\includegraphics[width=.49\columnwidth]{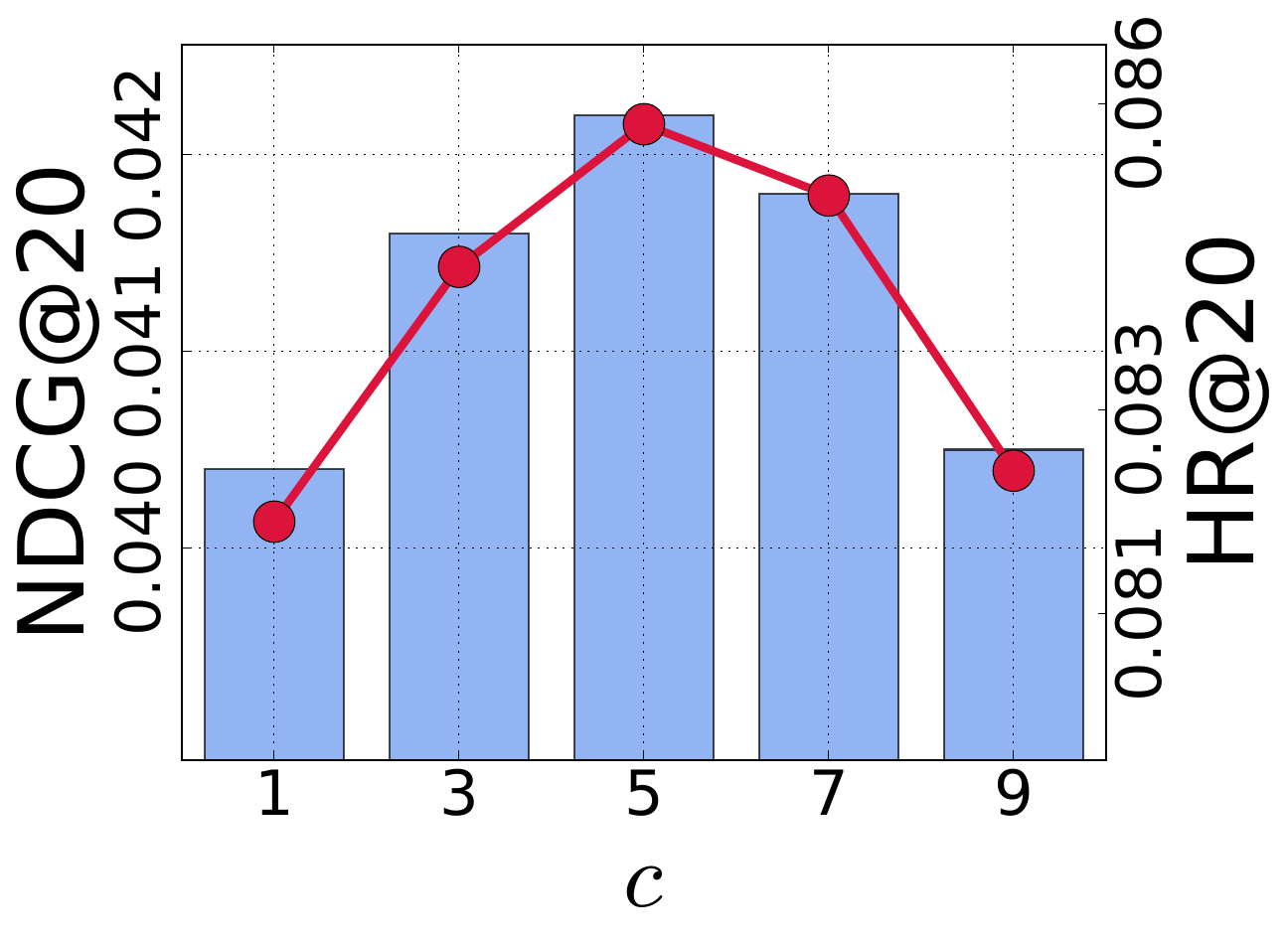}}
    \end{subfigure}
    \begin{subfigure}[Toys]{\includegraphics[width=.49\columnwidth]{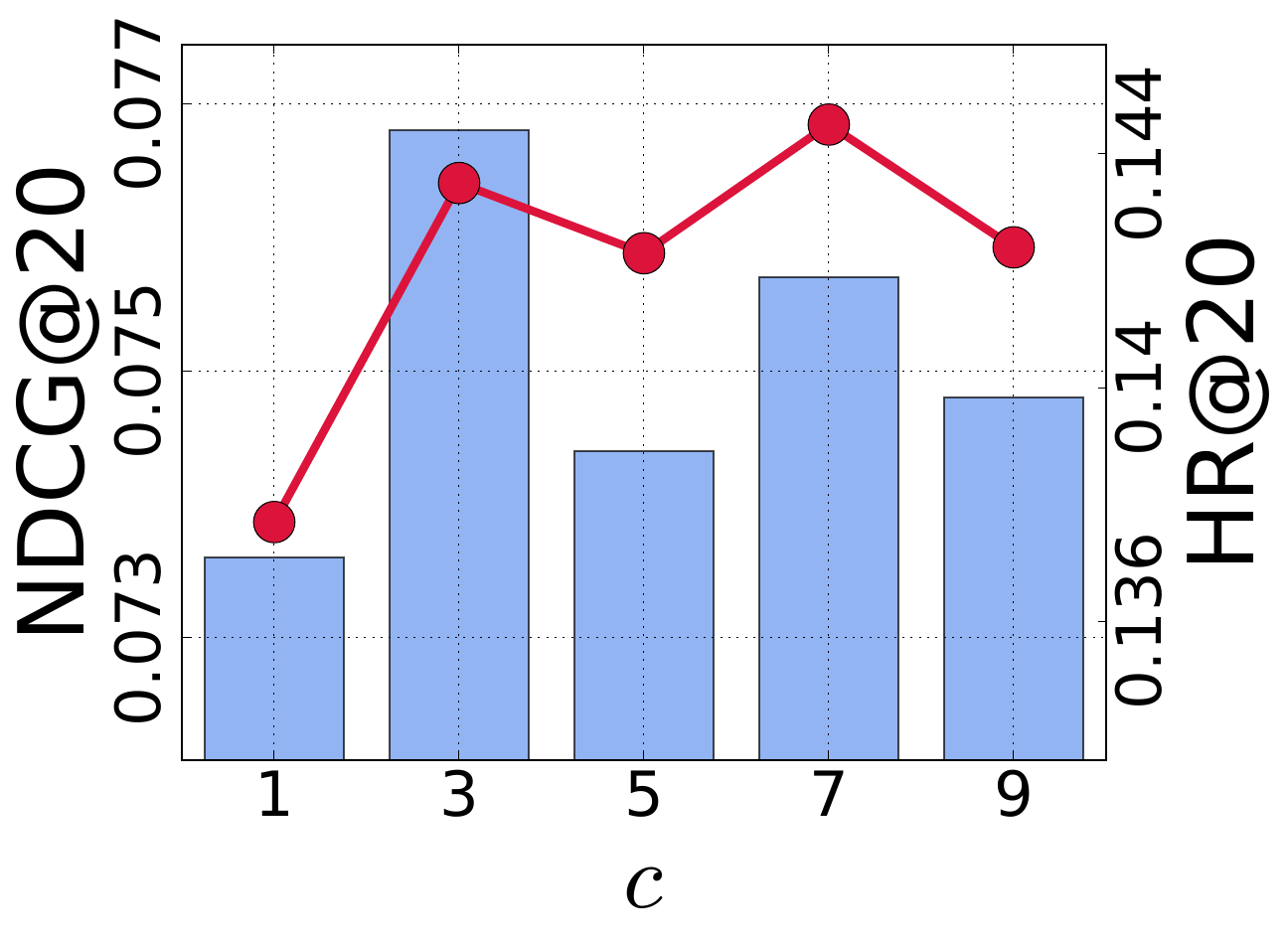}}
    \end{subfigure}
    \hspace{-1em}
    \begin{subfigure}[Yelp]{\includegraphics[width=.49\columnwidth]{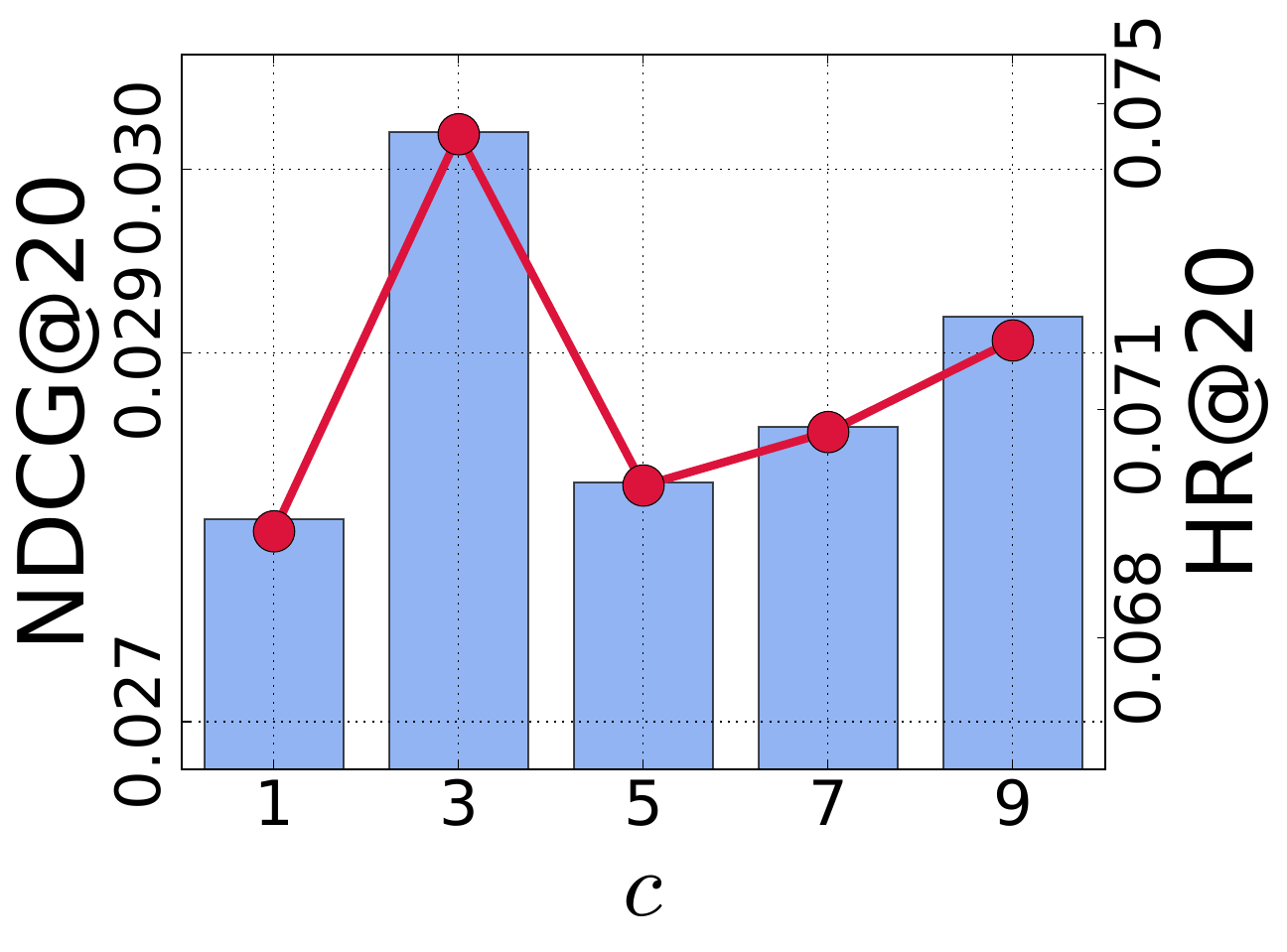}}
    \end{subfigure}
    \begin{subfigure}[LastFM]{\includegraphics[width=.49\columnwidth]{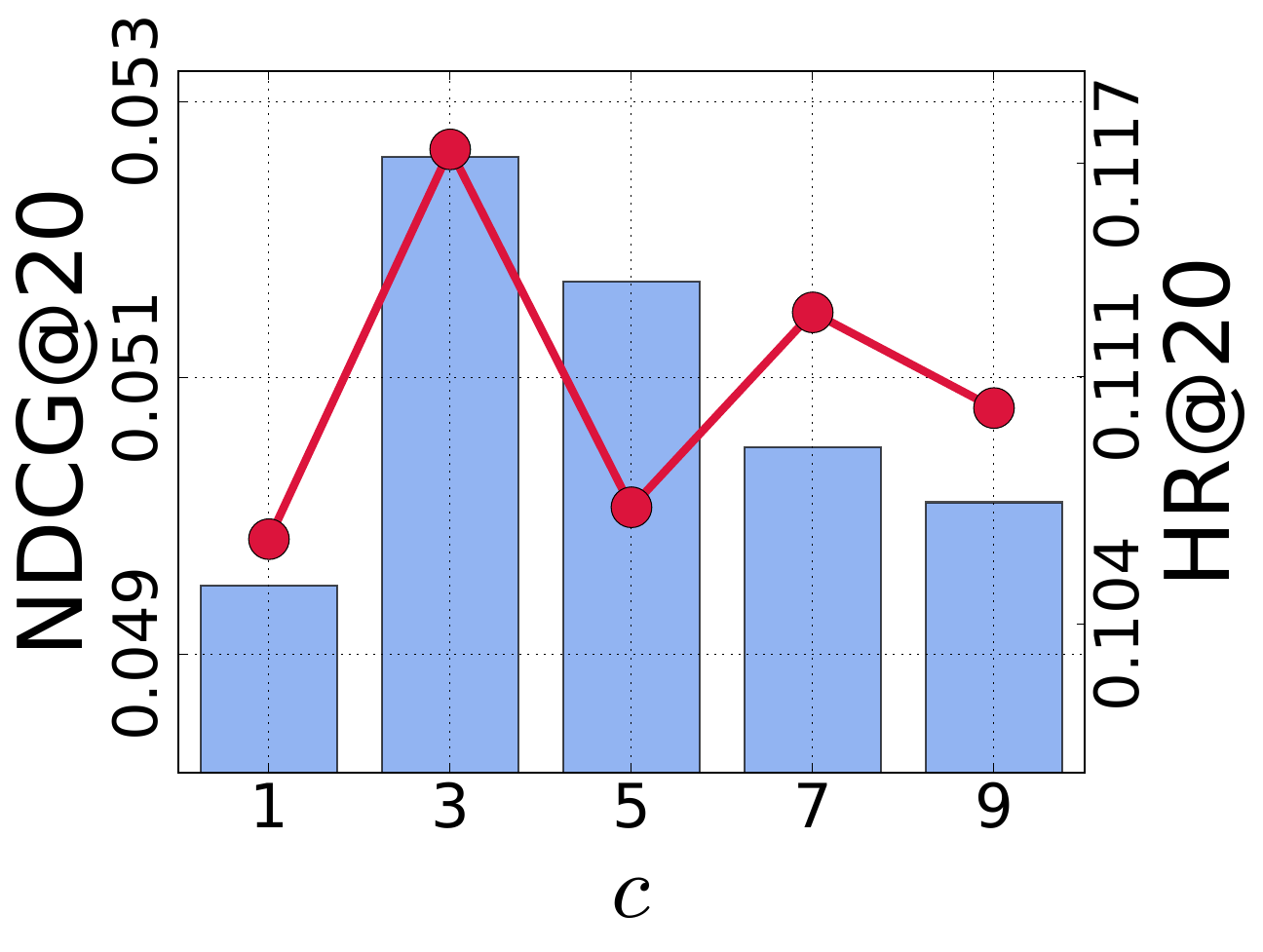}}
    \end{subfigure}
    \hspace{-1em}
    \begin{subfigure}[ML-1M]{\includegraphics[width=.49\columnwidth]{images/sensitivity_c_ML-1M.pdf}}
    \end{subfigure}
    \caption{Sensitivity to $c$ on all datasets}
    \label{fig:sen_c_appendix}
\end{figure}

\section{Model Complexity and Runtime Analyses}\label{app:runtime}
To evaluate the complexity and efficiency of BSARec, we evaluate the number of parameters and runtime per epoch. Results for all datasets are shown in Table~\ref{tab:runtime_appendix}. Overall, BSARec slightly increases the total parameters. However, if we use $\bm{\beta}$ as one scalar parameter $\beta$ in our model, the number of parameters may not make much difference. It only increases by 258 parameters compared to SASRec, DuoRec and FEARec. We show that BSARec has faster runtime times per epoch than FEARec and DuoRec across all datasets.

\begin{table*}[t]
    \small
    \centering
    \setlength{\tabcolsep}{1.5pt}
    \resizebox{1.02\textwidth}{!}{%
    \begin{tabular}{l cccccccccccccccccc}
    \toprule
    \multirow{2}{*}{Methods}    && \multicolumn{2}{c}{Beauty} && \multicolumn{2}{c}{Sports} && \multicolumn{2}{c}{Toys} && \multicolumn{2}{c}{Yelp} && \multicolumn{2}{c}{LastFM} && \multicolumn{2}{c}{ML-1M}  \\ 
    \cmidrule(lr){3-4} \cmidrule(lr){6-7} \cmidrule(lr){9-10} \cmidrule(lr){12-13} \cmidrule(lr){15-16} \cmidrule(lr){18-19}
                                && \# params & s/epoch  &&  \# params & s/epoch     && \# params & s/epoch      &&  \# params & s/epoch     && \# params & s/epoch      &&  \# params & s/epoch \\ 
    \midrule
    BSARec      && 878,208 & 12.75   && 1,278,592 & 18.58      && 866,880 & 11.63      && 1,385,856 & 21.20      && 337,088 & 3.11      && 322,368 & 20.73 \\
    \midrule
    SASRec      && 877,824 & 10.41   && 1,278,208 & 15.32      && 866,496 & 9.96      && 1,385,472 & 18.25      && 336,704 & 2.80      && 321,984 & 19.37 \\
    DuoRec      && 877,824 & 19.26   && 1,278,208 & 27.99      && 866,496 & 18.79     && 1,385,472 & 31.08      && 336,704 & 4.24     && 321,984 & 32.33 \\
    FEARec      && 877,824 & 156.83  && 1,278,208 & 233.42     && 866,496 & 132.43    && 1,385,472 & 257.56     && 336,704 & 27.82    && 321,984 & 278.24 \\
    \bottomrule
    \end{tabular}
    }
    \caption{The number of parameters and training time (runtime per epoch) on all datasets}
    \label{tab:runtime_appendix}
\end{table*}

\clearpage

\section{Performance under Different Setting}
In our main text, we use the evaluation strategy of DuoRec and FEARec, which employ all the items in the dataset. To further verify the effectiveness of our BSARec, we also conduct experiments under another strategy that SASRec and FMLP-Rec used. For evaluation, SASRec and FMLPRec pair the ground-truth item with 99 randomly sampled negative items that the user has not interacted with. 

We select Transformer-based SASRec as the representative baseline model and FMLP-Rec as one of the SOTA baselines. Table~\ref{tab:100sample} shows the results on Beauty, Sports, Toys, Yelp, LastFM, and ML-1M datasets. From the table, we can see similar tendencies as in Table~\ref{tab:main_exp}. BSARec performs much better than the other two methods. In most of the cases, BSARec outperforms SASRec and FMLP-Rec. These findings indicate that our beyond self-attention layer and attentive inductive bias with frequency rescaler are exactly effective for the SR task.
\newpage
\begin{table}[h]
    \centering
    \small
    \begin{tabular}{ll ccc}
    \toprule
    Dataset & Metric & SASRec & FMLPRec & BSARec \\
    \midrule
    \multirow{5}{*}{Beauty} & HR@5      & 0.3512 & 0.3922 & \textbf{0.4312} \\
                            & HR@10     & 0.4434 & 0.4914 & \textbf{0.5225} \\
                            & NDCG@5    & 0.2628 & 0.2964 & \textbf{0.3379} \\
                            & NDCG@10   & 0.2926 & 0.3284 & \textbf{0.3673} \\
                            & MRR       & 0.2637 & 0.2949 & \textbf{0.3350} \\
    \midrule
    \multirow{5}{*}{Sports} & HR@5      & 0.3480 & 0.3781 & \textbf{0.4133} \\
                            & HR@10     & 0.4717 & 0.4997 & \textbf{0.5303} \\
                            & NDCG@5    & 0.2492 & 0.2739 & \textbf{0.3102} \\
                            & NDCG@10   & 0.2891 & 0.3131 & \textbf{0.3479} \\
                            & MRR       & 0.2520 & 0.2742 & \textbf{0.3089} \\
    \midrule
    \multirow{5}{*}{Toys}   & HR@5      & 0.3594 & 0.3867 & \textbf{0.4224} \\
                            & HR@10     & 0.4566 & 0.4852 & \textbf{0.5180} \\
                            & NDCG@5    & 0.2726 & 0.2926 & \textbf{0.3351} \\
                            & NDCG@10   & 0.3040 & 0.3244 & \textbf{0.3659} \\
                            & MRR       & 0.2746 & 0.2917 & \textbf{0.3349} \\
    \midrule
    \multirow{5}{*}{Yelp}   & HR@5      & 0.5553 & 0.6058 & \textbf{0.6447} \\
                            & HR@10     & 0.7406 & 0.7707 & \textbf{0.7848} \\
                            & NDCG@5    & 0.3902 & 0.4337 & \textbf{0.4824} \\
                            & NDCG@10   & 0.4504 & 0.4873 & \textbf{0.5280} \\
                            & MRR       & 0.3748 & 0.4114 & \textbf{0.4587} \\
    \midrule
    \multirow{5}{*}{LastFM} & HR@5      & 0.2716 & 0.2853 & \textbf{0.3752} \\
                            & HR@10     & 0.3972 & 0.4138 & \textbf{0.5028} \\
                            & NDCG@5    & 0.1871 & 0.1975 & \textbf{0.2634} \\
                            & NDCG@10   & 0.2276 & 0.2394 & \textbf{0.3045} \\
                            & MRR       & 0.1976 & 0.2081 & \textbf{0.2636} \\
    \midrule
    \multirow{5}{*}{ML-1M}  & HR@5      & 0.6874 & 0.6763 & \textbf{0.7023} \\
                            & HR@10     & 0.7904 & 0.7858 & \textbf{0.7978} \\
                            & NDCG@5    & 0.5308 & 0.5212 & \textbf{0.5646} \\
                            & NDCG@10   & 0.5642 & 0.5568 & \textbf{0.5955} \\
                            & MRR       & 0.5020 & 0.4941 & \textbf{0.5406} \\
        \bottomrule
    \end{tabular}
    \vspace{-0.5em}
    \caption{Performance comparison on 99 negative sampling}
    \vspace{-1.5em}
    \label{tab:100sample}
\end{table}
\end{document}